\documentclass{article}

\usepackage{iclr2027_conference}
\iclrfinalcopy 
\usepackage[utf8]{inputenc}
\usepackage[T1]{fontenc}
\usepackage{times}
\usepackage{microtype}
\usepackage{nicefrac}
\usepackage{needspace}
\usepackage{float}
\usepackage{bbm}

\usepackage{soul}
\usepackage{amsmath,amssymb,amsthm,mathtools}
\usepackage{thmtools}
\usepackage{thm-restate}
\usepackage{bm}
\DeclareMathOperator{\E}{\mathbb{E}}

\newcommand{\abs}[1]{\lvert#1\rvert}

\newcommand{\grad}{\nabla}
\newcommand{\laplacian}{\Delta}

\newcommand{\defeq}{\coloneqq}

\newcommand{\ddt}{\frac{d}{dt}}

\newcommand{\trace}{\operatorname{tr}}

\renewcommand{\phi}{\varphi}

\providecommand{\DPS}{\mathrm{DPS}} 

\usepackage{xcolor}
\definecolor{xkcdWine}{HTML}{7f0000}

\usepackage{algorithm}
\usepackage{algpseudocode}

\usepackage{graphicx}
\usepackage{placeins}
\usepackage{hyperref}
\hypersetup{colorlinks=true,linkcolor=blue,citecolor=blue,urlcolor=blue}
\usepackage[nameinlink,capitalise]{cleveref}

\DeclareRobustCommand{\ola}[1]{\overleftarrow{#1}}
\DeclareRobustCommand{\ora}[1]{\overrightarrow{#1}}

\newcommand{\R}{\mathbb{R}}
\newcommand{\bbR}{\mathbb{R}}

\newcommand{\bbrd}{\mathbb{R}^d}

\newcommand{\Nc}{\mathcal{N}}
\newcommand{\Ac}{\mathcal{A}}
\newcommand{\x}{\mathbf{x}}
\newcommand{\y}{\mathbf{y}}
\newcommand{\z}{\mathbf{z}}
\newcommand{\s}{\mathbf{s}}
\newcommand{\rvx}{\mathbf{x}}

\theoremstyle{plain}
\newtheorem{lemma}{Lemma}

\theoremstyle{definition}

\newtheorem{remark}{Remark}

\crefname{theorem}{Theorem}{Theorems}
\Crefname{theorem}{Theorem}{Theorems}
\crefname{lemma}{Lemma}{Lemmas}
\Crefname{lemma}{Lemma}{Lemmas}
\crefname{proposition}{Proposition}{Propositions}
\Crefname{proposition}{Proposition}{Propositions}
\crefname{corollary}{Corollary}{Corollaries}
\Crefname{corollary}{Corollary}{Corollaries}
\crefname{definition}{Definition}{Definitions}
\Crefname{definition}{Definition}{Definitions}
\crefname{assumption}{Assumption}{Assumptions}
\Crefname{assumption}{Assumption}{Assumptions}
\crefname{remark}{Remark}{Remarks}
\Crefname{remark}{Remark}{Remarks}
\crefname{example}{Example}{Examples}
\Crefname{example}{Example}{Examples}
\crefname{algorithm}{Algorithm}{Algorithms}
\Crefname{algorithm}{Algorithm}{Algorithms}

\title{Diffusion-Based Posterior Sampling:\\ A Feynman--Kac Analysis of Bias and Stability}

\author{%
  William Porteous$^{1}$,\quad Advait Parulekar$^{2}$,\quad Sebastien Motsch$^{3}$,\\
  \textbf{Sanjay Shakkottai$^{2}$,\quad Matias G.~Delgadino$^{1}$}\\[0.4em]
  $^{1}$Department of Mathematics, UT Austin\\
  $^{2}$Department of Electrical and Computer Engineering, UT Austin\\
  $^{3}$School of Mathematical and Statistical Sciences, Arizona State University\\[0.4em]
  \texttt{\{wporteous, advaitp, sanjay.shakkottai, matias.delgadino\}@utexas.edu}\\
  \texttt{Sebastien.Motsch@asu.edu}
}

\begin{document}

\maketitle
\lhead{}
\renewcommand{\headrulewidth}{0pt}

\begin{abstract}

Diffusion-based posterior samplers draw from measurement- or reward-conditioned posteriors using pretrained diffusion priors. Though widely used for inverse problems, their output is biased even with exact prior scores. First, we quantify this bias for the DPS algorithm, the base method for a family of samplers, by introducing a surrogate curve of probability densities which interpolates between the true posterior and the standard Gaussian. The ratio of sampler to surrogate densities solves a parabolic PDE, and its Feynman--Kac formula yields an explicit representation of this ratio as the expectation of a path functional. This representation reveals which posterior regions are over- or under-sampled. This framework also provides a design principle for bias reduction that includes STSL, a related sampler.

Second, in hard-constraint problems (e.g.\ the low measurement noise regime) these samplers are known to be numerically unstable. We find that this instability of discretized DPS arises from an effective guidance update whose strength scales inversely with step size, inevitably violating the forward-Euler stability criterion near the constraint. This finding motivates implicit (backward-Euler) guidance updates. Finally, we show that the standard practice of early guidance-stopping, used to avoid oscillations, introduces a secondary bias which we characterize through our framework.

\end{abstract}


\section{Introduction}

Diffusion and score-based generative models~\citep{dickstein_2015,DDPM,song2020generativemodelingestimatinggradients,song2021score} have become the workhorse of modern generative modeling, powering text-to-image systems~\citep{rombach2022highresolutionimagesynthesislatent,dalle,imagen,dhariwal2021diffusion} and an expanding range of scientific and medical inverse problems~\citep{song2022solvinginverseproblemsmedical}. Their flexibility hinges on a single learned object (the score $\nabla\log\rho_t$ of the noised data distribution) which can be repurposed across downstream tasks without retraining.

A canonical such task is sampling from a posterior of the form $\mu_y(x)\propto e^{R_y(x)}\rho_*(x)$, encompassing both classical inverse problems $y=\mathcal{A}(x)+\varepsilon$ and reward-tilted generation~\citep{daras2024survey}. Even granting access to a perfect score oracle, posterior sampling is computationally intractable in the worst case~\citep{gupta}, so practical algorithms rely on heuristic guidance that approximates the time-dependent posterior score. An early and influential such heuristic is \emph{Diffusion Posterior Sampling} (DPS)~\citep{chung2023diffusion}: it replaces the intractable conditional score $\nabla_{x_t}\log p(y\mid x_t)$ by the gradient of the reward evaluated at the Tweedie posterior mean $\hat{x}_t(x_t)=\E[X_0\mid X_t=x_t]$~\citep{tweedie}, yielding a plug-and-play guidance compatible with any pretrained score network. Its simplicity has made DPS the de facto baseline for inverse problems and inspired a line of research targeting its known weaknesses: manifold-constrained gradients~\citep{chung2022improving}, denoising restoration~\citep{kawar2022denoising}, pseudoinverse-guided diffusion~\citep{song2023pseudoinverse}, latent-space extensions~\citep{psld},   second-order Tweedie corrections~\citep{STSL,boys2024tweediemomentprojecteddiffusions}, proximal approaches to decrease the gradient computation burden~\citep{rout2024rbmodulationtrainingfreepersonalizationdiffusion}, filtering and SMC-based reweightings~\citep{dou2024diffusion,wu2024practicalasymptoticallyexactconditional,moufad2025variational}, and recent drift-control schemes~\citep{driftlite,guo2026conditional,anil2026finetuningdiffusionmodelsintermediate}.

Yet despite this flurry of activity, two basic questions remain open. First, the DPS approximation is biased even for Gaussian-mixture priors with quadratic rewards; but \emph{which} samples does this bias over- or under-represent, and \emph{why} do correctives like STSL~\citep{STSL} improve performance? Existing analyses establish convergence under restrictive assumptions on the prior or measurement operator~\citep{xu2024provablyrobustscorebaseddiffusion,parulekar2025efficientapproximateposteriorsampling,moitra2026steeringdiffusionmodelsquadratic} or treat the algorithm as a black box, leaving its preferred classes unexplained. Second, in the low-temperature regime needed for hard measurement constraints in image inverse problems, standard DPS is numerically unstable. Practitioners routinely fall back on early guidance-stopping and trajectory-dependent step sizes, but the effect of these heuristics on the sampled distribution has never been quantified.

\textbf{Contributions.} We close both gaps with a unified analysis based on the classical Feynman--Kac formula~\citep{karatzas}, recently used to correct guided and annealed diffusion sampling via SMC reweighting~\citep{skreta2025,akhound2025}, and complementing stochastic-analytic perspectives on guidance~\citep{bruna2024posteriorsamplingdenoisingoracles,driftlite,guo2026conditional}.

\textbf{(i) An exact bias formula for DPS.} In \cref{sec:bias}, we compare DPS with a surrogate path of densities between the standard Gaussian and the true posterior; their ratio solves a parabolic PDE, whose Feynman--Kac formula expresses the DPS bias exactly as an expectation over Ornstein--Uhlenbeck paths (\cref{thm:DPSbias}). The reaction term in this expectation compares conditional covariance with reward curvature, revealing where DPS over- or under-samples, and gives a design principle for bias reduction that includes STSL~\citep{STSL} (\cref{sec:reducing_the_variance}).

\textbf{(ii) The source of DPS instability.} In \cref{sec:stability}, we show that in hard-constraint (low-temperature) problems, the instability of discretized DPS arises from an effective guidance update whose strength scales inversely with the step size: the guidance schedule cancels the Euler step, so the forward-Euler stability criterion is inevitably violated near the constraint, and no choice of step size repairs it. This motivates implicit (backward-Euler) guidance updates, which we illustrate in \cref{app:instability}.

\textbf{(iii) The bias of early guidance-stopping.} Practitioners mitigate these oscillations with early guidance-stopping, see \cref{alg:early_stopping}. Using the framework of (i), we show that this introduces a secondary bias, and characterize the resulting output as an explicit reweighting of the prior (\cref{thm:early}).

\section{Background and Related Work}\label{sec:background}
\textbf{Score Based Generative Models.} We consider sampling from a density $\rho_*$. A subscript $t$ denotes a noised distribution: $\rho_t \coloneqq \big(e^{td}\rho_*(e^{t}\,\cdot\,)\big)*\mathcal{N}(0, \sigma_t^2 I)$, with $\sigma_t^2 \coloneqq 1-e^{-2t}$, is the marginal of the Ornstein--Uhlenbeck (OU) noising process
\begin{equation}\label{eq:SDE}
dX_t = -X_t\,dt+\sqrt{2}\,dB_t,\qquad X_0\sim\rho_*,
\end{equation}
which interpolates between $\rho_*$ at $t=0$ and the standard Gaussian $\gamma=\mathcal{N}(0,I)$ as $t\to\infty$. 
Score based generative models sample by reversing \eqref{eq:SDE} on a horizon $[0,T]$: the reversed process $\ola X_t \coloneqq X_{T-t}$, with marginals $\ola\rho_t = \rho_{T-t}$, solves the \emph{denoising} SDE
\begin{equation}\label{eq:reverse_OU}
d\ola X_t = \big(\ola X_t + 2\nabla \log \rho_{T-t}(\ola X_t)\big)\,dt + \sqrt{2}\,d\ola B_t, \qquad \ola X_0 \sim\rho_T,
\end{equation}
where $\ola B$ is a Brownian motion in the filtration of $\ola X$ on $[0,T)$ \citep{haussmann1986time,anderson}. In practice, $\rho_T$ is replaced by $\gamma$ and the score by a network $s_{\theta_*}(x,t)\approx\nabla\log\rho_t(x)$, trained from samples $X\sim\rho_*$ by denoising score matching \citep{vincent2011connection} in the continuous-time form of \citet{song2021score}, of which \eqref{eq:SDE} is the variance-preserving case:
$\theta_* = \operatorname{argmin}_\theta \E_{t,X,\xi}\big[\sigma_t^2\,\|s_\theta(e^{-t}X+\sigma_t\xi,\, t)+\xi/\sigma_t\|^2\big]$, with $\xi\sim\mathcal{N}(0,I)$ independent of $X$ and $t\sim\mathrm{Unif}(0,T]$. This is minimized by $s(\cdot,t)=\nabla\log\rho_t$, and coincides with the DDPM objective $\E\|\epsilon_\theta - \xi\|^2$ \citep{DDPM} under $\epsilon_\theta \coloneqq -\sigma_t s_\theta$. By Tweedie's formula, the network also gives the denoiser $\E[X_0\mid X_t = x] = e^{t}\big(x + \sigma_t^2\nabla\log\rho_t(x)\big)\approx e^{t}\big(x + \sigma_t^2\, s_{\theta_*}(x,t)\big)$.
Sampling via this reverse process runs in polynomial time \citep{rombach2022highresolutionimagesynthesislatent, song2020generativemodelingestimatinggradients, dalle, imagen}, which compares favorably to classical approaches such as Langevin dynamics \citep{vempala2022rapidconvergenceunadjustedlangevin}, whose convergence rate is instance-dependent and can be arbitrarily slow, see \cref{sec:relentropy}.

\textbf{Posterior Sampling.} A natural application of score-based models is to inverse problems and posterior sampling. The score network characterizes a prior $\rho_*$, and at test time one tilts the samples by a log-likelihood $R_y(x)$ to target the posterior $\mu_y \coloneqq e^{R_y}\rho_*/Z$; the subscript records the dependence on the measurement $y$, e.g.\ $R_y(x) = -\frac{\alpha}{2}\|\mathcal{A}(x)-y\|^2$ for $y=\mathcal{A}(x)+\varepsilon$ with $\varepsilon\sim\mathcal{N}(0,\alpha^{-1}I)$. The main challenge is that the posterior score $\nabla \log (\mu_y)_t$ (the noised posterior is denoted $\ora{\mu}^{\ast}_t$ in \cref{sec:bias}) cannot be easily computed from $\nabla \log \rho_t$.\footnote{The noising process and the tilt do not commute: $(e^{R_y}\rho_*)_t \neq (e^{R_y})_t \rho_t$.} A range of approximate algorithms have been proposed to circumvent this. A central theme is the use of the prior scores $\nabla \log \rho_t$ through Tweedie's formula to obtain realistic-looking samples even when posterior sampling. Specifically, $\E[X_0 \mid X_t = x_t]$ is used\footnote{By Tweedie's formula, $\E[X_0 \mid X_t = x_t] = e^{t}x_t + (e^{t}-e^{-t})\nabla\log\rho_t(x_t)$, which can be efficiently estimated at test-time.} as a computationally tractable proxy for the initial condition $X_0$ (i.e., the mean endpoint of the reverse diffusion \eqref{eq:reverse_OU} run to completion from $x_t$) and is fed into the reward model $R_y(\cdot)$ when modifying the drift at test time, see \cref{sec:Tweedies}. Although the resulting samples are not formally drawn from $\mu_y$, this heuristic performs well in practice.

\textbf{Feynman--Kac formulas.} The sampling literature often focuses on error bounds for approximate sampling algorithms \citep{lee2023convergencescorebasedgenerativemodeling, chen2023samplingeasylearningscore, vempala2022rapidconvergenceunadjustedlangevin}, see \cref{sec:relentropy}. These are instantiated as upper bounds on the KL/TV/$\chi^2$ distance between the distribution of the sampling algorithm and the ground truth, and under favorable circumstances can be shown to be polynomially or exponentially small in the parameters of the instance. In posterior sampling, such an error is {\em known} to be large.
As such, a KL bound is often vacuous, unless it is accompanied by strong assumptions about the instance. Rather than focusing on bounding this error, we apply machinery that allows us to explicitly track the Radon--Nikodym derivative of approximate posterior sampling algorithm with respect to the true posterior.

In particular, we will exploit the Feynman--Kac representation. Consider two time-dependent densities evolving under possibly different transport and reaction fields:

\begin{equation}\label{eq:fp-pair-body}
\begin{cases}
    \partial_t \pi_t = -\nabla\cdot(v_t\pi_t) + \Delta \pi_t + f_t\pi_t\\[2pt]
    \partial_t \pi'_t = \underbrace{-\nabla \cdot (v'_t\pi'_t)}_{\text{transport}}+\underbrace{\Delta\pi'_t}_{\text{diffusion}}+\underbrace{f'_t\pi'_t}_{\text{reaction}}
\end{cases}
\end{equation}
Without the reaction term, \eqref{eq:fp-pair-body} are Fokker--Planck equations, whose solutions are the laws of SDEs with drift $v_t$ (resp.\ $v'_t$). The reaction term reweights paths by $\exp\big(\int_0^t f_s(X_s)\,ds\big)$, i.e.\ killing where $f_t<0$ and branching where $f_t>0$, so $\pi_t$ is an unnormalized measure whose normalization is the law of the weighted process.

Assuming $\pi_t > 0$, the \emph{ratio} of the marginals $g_t \coloneqq \pi'_t/\pi_t$ satisfies
\begin{equation}\label{eq:ratio-pde}
    \partial_t g_t \;=\; \Delta g_t \;+\; b_t\cdot\nabla g_t \;-\; c_t\,g_t,
\end{equation}
for an explicit choice of $b_t, c_t$ given in \cref{lem:FK_ratio}. Fix $t>0$ and let $(A_s)_{s\in[0,t]}$ be the diffusion process associated to the stochastic characteristics
\begin{equation}\label{eq:fk-sde}
    dA_s = b_{t-s}(A_s)\,ds + \sqrt{2}\,dB_s,
\end{equation}
which run backward through the coefficients of \eqref{eq:ratio-pde}. The Feynman--Kac representation of \eqref{eq:ratio-pde} then reads
\begin{equation}\label{eq:fk}
    g_t(x) = \E\!\left[g_0(A_t)\exp\!\left(-\int_0^t c_{t-s}(A_s)\,ds\right)\bigg|\,A_0 = x\right].
\end{equation}
See \cref{app:FK_background} for the precise statement and proof, and see \citet{skreta2025}, \citet{akhound2025}, and \citet{driftlite} for the use of Feynman--Kac 
formulae in posterior sampling and annealing with diffusion models.

\section{The Bias of DPS via Surrogate Paths}\label{sec:bias}
This section introduces a surrogate path framework for analyzing diffusion-based posterior samplers. Given a reward $R_y:\R^d\to \R$ our goal is to sample from the posterior $\mu_y(x)= \exp(R_y(x))\rho_*(x)/Z$, an exponential tilt of the prior where $Z > 0$ is the normalization. To this end, we define a {\em surrogate path} to be a continuous curve of probability measures (marginal densities) $\overrightarrow{\mu}:[0,\infty)\to \mathcal{P}(\R^d)$ of the form $t\mapsto \overrightarrow{\mu}_t\coloneqq h_t\rho_t/ Z_t$ where the \emph{positive} function $h:[0,\infty)\times \R^d\to (0,\infty)$ is a weight satisfying $h \in C^{1,2}([0,\infty)\times\bbR^d)$, $h_0(x)=\exp(R_y(x))$ (so that $Z_0=Z$) and $h_\infty$ identically constant. A valid surrogate thus interpolates between $\mu_y$ at time zero and the standard Gaussian $\gamma$ as $t \to \infty$. For DPS, the sampler density is the marginal law of an associated SDE, which we recover by dropping the reaction term from the exact evolution of the surrogate. This enables us to express the ratio of surrogate to sampler marginals via Feynman--Kac formula. First, we record the evolution of a general surrogate.
\begin{lemma}[Informal, see \cref{lem:tilted_FP}]\label{lem1}
    For any finite horizon $T< \infty$ and surrogate path $\ora{\mu}_t =  h_t \rho_t/Z_t$, its time-reversal $\overleftarrow{\mu}_t=\overrightarrow{\mu}_{T-t}$ satisfies
\begin{equation}\label{surrogate}
\bigg\{\begin{aligned}
\partial_t \overleftarrow{\mu}_t
&= - \nabla \cdot (x \overleftarrow{\mu}_t)
   - 2\nabla \cdot (\nabla\log\ora{\mu}_{T-t} \overleftarrow{\mu}_t)+ \Delta \ola{\mu}_t
   - c[h_{T-t},\rho_{T-t}]\ola{\mu}_t\\
   \ola\mu_0&=\ora\mu_T
\end{aligned}
\tag{Surrogate Path}
\end{equation}
where, denoting the generator of the reverse OU process \eqref{eq:reverse_OU} by $\ola{L}_t^{OU}$ (see \cref{eq:OU_backward_generator}), the reaction term is 
\begin{equation*}
c[h_{T-t},\rho_{T-t}](x) = -\frac{1}{h_{T-t}(x)}\big[(\partial_t + \ola{L}^{OU}_{t})\,h_{T-t}\big](x) + \ddt \log Z_{T-t}.
\end{equation*} 
\end{lemma}
We suppose access to $\grad \log\rho_t(x)$ via pretrained score network and also access to the reward and its spatial gradient, computing $\nabla\log\ora\mu_{t}=\nabla\log h_t+\nabla\log\rho_t$. We now introduce the notion of the {\em algorithm path} $\ola{\nu}_t$ with horizon $T$ associated to the surrogate path $\ola{\mu}_t$, solving the reactionless equation
    \begin{equation}\label{algorithm}
        \bigg\{\begin{aligned}
        \partial_t \overleftarrow{\nu}_t &=- \nabla \cdot (x \overleftarrow{\nu}_t)-2\nabla \cdot (\nabla\log\ora{\mu}_{T-t} \overleftarrow{\nu}_t) +\Delta \ola{\nu}_t \\
        \ola{\nu}_0&=\gamma.
        \end{aligned}
        \tag{Algorithm Path}
    \end{equation}
We find $\ola\nu_t=\mathrm{Law}(Y_t)$ for the associated SDE
\begin{equation}\label{eq:algSDE}
    \begin{cases}
        dY_t= \Big[Y_t+2\nabla\log\ora{\mu}_{T-t}(Y_t)\Big]dt+\sqrt{2}\,dB_t,\\
        Y_0\sim\gamma.
    \end{cases}
    \tag{Algorithm SDE}
\end{equation} 
The difference between \eqref{surrogate} and \eqref{algorithm} is simply presence of the reaction term. By a Feynman--Kac argument (see \cref{thm:DPSbias}(i) below), we can express their ratio as a weighted expectation over the paths \eqref{eq:algSDE}:
\[
\frac{\ola{\mu}_T(x)}{\ola{\nu}_T(x)}\approx\E_{Y\sim\eqref{eq:algSDE}}\left[\exp\left(-\int_0^Tc[h_{T-t},\rho_{T-t}](Y_t)\;dt\right)\bigg|Y_T=x\right],
\]
where the approximation is only the boundary approximation $\ora\mu_T(Y_0)/\gamma(Y_0)\approx 1$ (valid at large T as $h_{\infty}$ is constant). Thus, the reaction term encodes the gap between the output of the algorithm $\mathrm{Law}(Y_T)$ and the true posterior $\mu_y$. Note that due to this modification in \eqref{algorithm}, some amount of bias is unavoidable. Indeed in the worst case, any path beginning at the posterior $\mu_y$ and ending in a tractable distribution like $\gamma$, is generated by an evolution that we cannot compute in polynomial time, see \cref{rem:perfect} and \citep{gupta}. Nevertheless, some paths inspire useful approximations, that yield good empirical results; see for instance \citep{bruna2024posteriorsamplingdenoisingoracles,parulekar2025efficientapproximateposteriorsampling,driftlite}.

\textbf{Canonical Interpolation.}\;
A canonical choice of \eqref{surrogate} $\ora{\mu}^{\ast}$ is induced by the Doob h-transform \citep{doob1984},
\begin{equation}\label{eqn:h-transform} 
\overrightarrow{\mu}^{\ast}_t(x) \coloneqq \frac{h^{\ast}_t(x)\rho_t(x)}{Z_t}, \quad h^{\ast}_t(x) \coloneqq \E\!\left[e^{R_y(X_0)}\,\bigm|\,X_t=x\right], 
\end{equation}
where $X_t$ is the OU process \eqref{eq:SDE}. This is the unique tilt $h^{\ast}_t$ satisfying $h^{\ast}_0(x) = e^{R_y(x)}$ that also enjoys null reaction term $c[h^{\ast}_t,\rho_t]\equiv 0$ in \eqref{surrogate}. This zero reaction term implies that $\ola{\mu}^{\ast}_t=\ora{\mu}^{\ast}_{T-t}$ is exactly the marginal law of the denoising SDE \eqref{eq:algSDE} with $\ora{\mu}_{T-t}=\ora{\mu}^{\ast}_{T-t}$. 
We call this induced path $\ora{\mu}^{\ast}$ the \emph{canonical interpolation} as it is equal (by Bayes' rule) to the conditional law, $\ora{\mu}^{\ast}_t=\mathrm{Law}(X_t \vert \mathcal{A}(X_0)+\varepsilon = y)$ in the measurement case, and retains the correct endpoints $\ora{\mu}^{\ast}_0=\mu_y$ and $\ora{\mu}^{\ast}_\infty=\gamma$. To simulate $\ora{\mu}^{\ast}$ via \eqref{eq:algSDE} requires access to the score $\grad\log\ora{\mu}^{\ast}_t=\grad\log\rho_t+\grad\log h^{\ast}_t$. While the score network supplies the first term, the latter requires us to solve for an optimal drift control \citep{flemingsoner}; this is expensive and must be done for each new choice of reward. For approximations to this optimal drift, see \citep{driftlite} for an online technique and \citep{guo2026conditional} for an offline training method. 
\noindent\fbox{%
  \parbox{\dimexpr\linewidth-2\fboxsep-2\fboxrule\relax}{%
    \emph{Design Problem: identify a surrogate path
    \(t \mapsto \overrightarrow{\mu}_t\) whose algorithm SDE is readily computable for any choice of differentiable reward, and yet also enjoys a small reaction term $c[h_t,\rho_t]$.}%
  }%
}
\begin{figure}[t]
\centering
\includegraphics[width=0.7\textwidth]{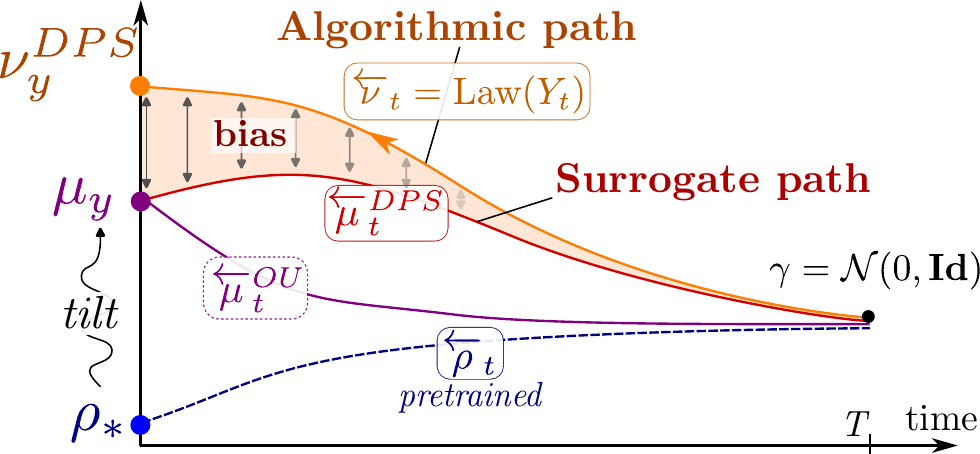}
\caption{Paths between the endpoints $\rho_*$ or $\mu_y$ (left) and $\gamma$ (right). {\color{blue}Blue, dotted}: the forward OU path $\rho_t$ and its reversal $\ola{\rho}_t$. {\color{violet}Violet}: the canonical interpolation $\ora{\mu}^{\ast}_t$ (OU started at $\mu_y$), whose reversal is intractable at inference time. {\color{red}Red}: the DPS surrogate $\ora\mu^{\DPS}_t=e^{R_y(\hat{x}_t)}\rho_t/Z_t$, with the same endpoints. {\color{orange}Orange}: the DPS algorithm path $\ola{\nu}^{\DPS}_t$ which  ends at the biased $\nu^{\DPS}_y$.}

\label{fig:sketch}  
\end{figure}

\textbf{Diffusion Posterior Sampling.}\; {\color{black} We show below that the 
DPS algorithm \citep{chung2023diffusion}} can be interpreted as stemming from the following \eqref{surrogate} (red line in \cref{fig:sketch}).
\begin{equation}\label{eq:DPScurve}
\overrightarrow{\mu}^{\DPS}_t(x)=\frac{1}{Z_t}\,e^{R_y(\hat{x}_t(x))}\rho_t(x),\tag{DPS Surrogate Path}
\end{equation}
which retains the correct endpoints $\overrightarrow{\mu}^{\DPS}_0 =\mu_y$, $\overrightarrow{\mu}^{\DPS}_\infty = \gamma$. Crucially, the score $\nabla\log\ora\mu_t^{\DPS}$ can be written in terms of
\[
\hat{x}_t(x)\coloneqq\E[X_0\mid X_t=x],\qquad \Sigma_t(x)\coloneqq\mathrm{Cov}(X_0\mid X_t=x),
\]
both of which are computable from the diffusion score oracle $s_{\theta_*}$ and its Jacobian $\nabla s_{\theta_*}$ via Tweedie's formula \citep{tweedie}, see \cref{sec:identities}. Heuristically, the \eqref{eq:DPScurve} is obtained from the canonical interpolation by swapping the conditional expectation \emph{inside} the exponential \eqref{eqn:h-transform} for 
\[
h^{\DPS}_t(x)=e^{R_y(\E[X_0\mid X_t=x])}\ne h^{\ast}_t(x)=\E[e^{R_y(X_0)}\mid X_t=x].
\]

We can identify the true reversal of the \eqref{eq:DPScurve},
\begin{equation}\label{eq:truereversal}\tag{DPS Surrogate PDE}
\partial_t\overleftarrow{\mu}^{\DPS}_t=\Delta\overleftarrow{\mu}^{\DPS}_t-2\,\nabla\!\cdot\!\big(\nabla\log\overrightarrow{\mu}^{\DPS}_{T-t}\,\overleftarrow{\mu}^{\DPS}_t\big)-\nabla\!\cdot(x\,\overleftarrow{\mu}^{\DPS}_t)-c_{\DPS}(T-t,x)\,\overleftarrow{\mu}^{\DPS}_t,
\end{equation}
with an explicit reaction coefficient
\begin{equation}\label{define_c_DPS}
c_{\DPS}(t,x)= -\frac{\trace\!\Big(\Sigma_t(x)(D^2 R_y)(\hat{x}_t(x))\Sigma_t(x)\Big)+\big\|\Sigma_t(x)\nabla R_y(\hat{x}_t(x))\big\|^2}{(e^t-e^{-t})^2}-\frac{d}{dt}\log Z_t.
\end{equation}

\textbf{The DPS algorithmic path.}\;
The difficulty with implementing Equation~\eqref{eq:truereversal} as an SDE is the reaction term. We can construct an alternate PDE that has the same transport and diffusion term but no reaction term
\begin{equation}\label{eq:DPSreversalPDE}\tag{DPS Path}
\partial_t\overleftarrow{\nu}^{\DPS}_t=\Delta\overleftarrow{\nu}^{\DPS}_t-2\,\nabla\!\cdot\!\big(\nabla\log\overrightarrow{\mu}^{\DPS}_{T-t}\,\overleftarrow{\nu}^{\DPS}_t\big)-\nabla\!\cdot(x\,\overleftarrow{\nu}^{\DPS}_t).
\end{equation}
Using the identities of \cref{sec:identities}, for the score $\nabla\log\ora\mu_{t}^{\DPS}$, we get the \eqref{eq:algSDE} approximated by the DPS algorithm as
\begin{equation}\label{eq:DPS_SDE}\tag{DPS SDE}
\begin{cases}
dY_t = \big(Y_t+2\nabla\log\rho_{T-t}(Y_t)+\tfrac{2}{e^{T-t}-e^{-(T-t)}}\,\Sigma_{T-t}(Y_t)\,\nabla R_y(\hat{x}_{T-t}(Y_t))\big)\,dt+\sqrt{2}\,dB_t,\\[2pt]
Y_0\sim \gamma,
\end{cases}
\end{equation}

Applying the Feynman--Kac formula \eqref{eq:fk} to the quotient $\frac{\ola\nu^{\DPS}_T}{\ola\mu^{\DPS}_T}$, we obtain the following characterization of the bias of the DPS algorithm.

\begin{restatable}{theorem}{mainthm}\label{thm:DPSbias}
    The terminal law $\nu^{\DPS}_y \coloneqq \overleftarrow{\nu}^{\DPS}_T$ of the DPS-SDE \eqref{eq:DPS_SDE} differs from the true posterior $\mu_y$ by a pointwise multiplicative weight:
\begin{equation}\label{eq:fk_DPS}
  \mu_y(x) = \omega(x)\,\nu^{\DPS}_y(x).
\end{equation}
The weight $\omega$ admits two equivalent Feynman--Kac representations in terms of $c_{\DPS}$ defined in \eqref{define_c_DPS}\textup{:}
\begin{enumerate}
  \item[\textnormal{(i)}] \textbf{Backward path} (condition on the DPS denoising process arriving at $Y_T = x$)\textup{:}
  \begin{equation}\label{eq:weight_DPS_backward}
    \omega(x) = \E_{Y\sim \eqref{eq:DPS_SDE}}\!\left[\frac{\overrightarrow{\mu}^{\DPS}_T(Y_0)}{\gamma(Y_0)}\,\exp\!\Bigl({-\int_{0}^{T}c_{\DPS}(T-s,Y_s)\,ds}\Bigr)\,\Bigm|\,Y_T=x\right].
  \end{equation}
  \item[\textnormal{(ii)}] \textbf{Forward path} (condition on the OU process \eqref{eq:SDE} starting at $X_0 = x$)\textup{:}
  \begin{equation}\label{eq:weight_DPS_forward}
    \frac{1}{\omega(x)} = \E_{X\sim\mathrm{OU}}\!\left[\frac{\gamma(X_T)}{\overrightarrow{\mu}^{\DPS}_T(X_T)}\,\exp\!\Bigl({\int_{0}^{T}c_{\DPS}(s,X_s)\,ds}\Bigr)\,\Bigm|\,X_0=x\right].
  \end{equation}
\end{enumerate}
Both path functionals are expressible in terms of quantities obtainable from the score oracle and its Jacobian via Tweedie's formula.
Importance-weighting DPS samples by $\omega$ recovers $\mu_y$ exactly.
\end{restatable}

\begin{figure}[t]
  \centering
  \includegraphics[width=\linewidth]{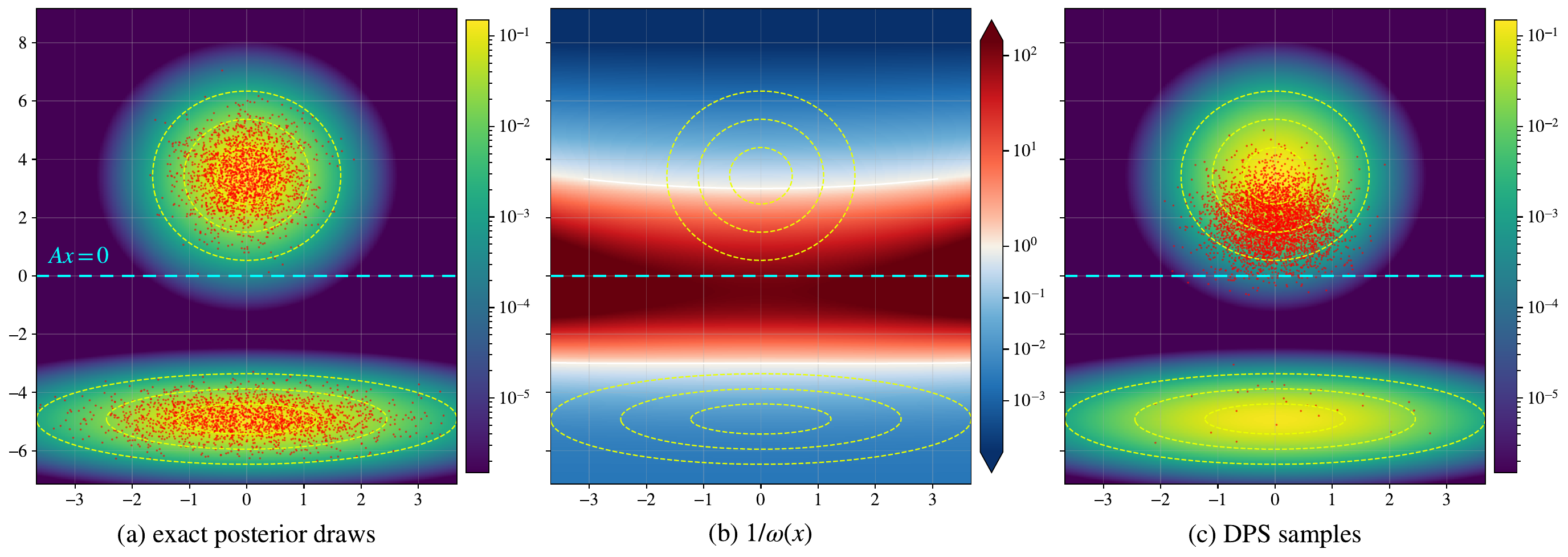}
\caption{Two-mode Gaussian mixture prior; constraint $\{x_2=0\}$ is the dashed line; background in all panels is the exact posterior density $\mu_y$.
\textbf{(a)} Draws from the true posterior $\mu_y$ (red). Mode shares $\approx 50\%/50\%$.
\textbf{(b)} The ratio $1/\omega=\nu^{\DPS}_y/\mu_y$, log scale, from a finite-difference solve of \eqref{eq:psi_pde_fig}. Red: $1/\omega>1$, DPS over-samples. Blue: $1/\omega<1$, DPS under-samples. White contour: $1/\omega=1$.
\textbf{(c)} Draws generated by \eqref{eq:DPS_SDE} (red). Mode shares $\approx 99\%/1\%$.
Two effects are visible in (c) and predicted by (b): the mode nearer the constraint receives $10^2$ times more samples than the other, and inside each mode the samples are shifted toward $x_2=0$ relative to (a). Shares are computed from all $N$ samples. Setup and solver: \cref{app:sampling_figure}.}
\label{fig:sampling_figure}
\end{figure}

See \cref{app:proof_DPSbias} for the proof of \cref{thm:DPSbias}.
\paragraph{Discussion of \cref{thm:DPSbias}.}
By~\eqref{eq:fk_DPS}, $\nu^{\DPS}_y=\omega^{-1}\mu_y$: DPS over-samples where $\omega<1$ and under-samples where $\omega>1$. For a linear-Gaussian reward $R_y(x)=-\tfrac{\alpha}{2}\|\mathcal Ax-y\|^2$, \eqref{define_c_DPS} and~\eqref{eq:tweedie_cov} give, with $J_s\coloneqq\nabla\hat x_s$ and $r_s\coloneqq\mathcal A\hat x_s-y$,
\begin{equation}\label{eq:cdps_quadratic}
c_{\DPS}(s,x)=\alpha\|\mathcal AJ_s\|_F^2-\alpha^2\|J_s\mathcal A^\top r_s\|^2-\tfrac{d}{ds}\log Z_s,
\end{equation}
which is largest when the Tweedie mean satisfies the constraint and decreases with the residual. In~\eqref{eq:weight_DPS_forward}, $c_{\DPS}$ is integrated along the forward OU path from $x$. Since $\|J_s\|$ decays for $s\gtrsim1$, small $s$ dominates, where $X_s\approx x$ and $r_s\approx r\coloneqq\mathcal Ax-y$; substituting,
\begin{equation}\label{eq:omega_quadratic}
\frac1{\omega(x)}\approx C\,e^{-\alpha^2 r^\top K(x)\,r},
\qquad
K(x)\coloneqq\int_0^T\mathcal (\mathcal AJ_s)(\mathcal AJ_s)^\top ds\succeq0 .
\end{equation}
DPS therefore samples $e^{R_y}\rho_*$ with the likelihood precision $\alpha I$ replaced by $\alpha I+2\alpha^2K(x)$: the measurement is treated as less noisy than it is, most strongly along measurement directions tangent to the data manifold, where $J_s\to P_{T\mathcal M}$ (\cref{sec:zeronoiselimit}).

In \cref{fig:sampling_figure}, $\mathcal A=e_2^\top$, $y=0$, $\alpha=0.4$, so $K$ is the scalar $\kappa(x)=\int_0^T\|\partial_{x_2}\hat x_s(x)\|^2ds$ and~\eqref{eq:omega_quadratic} reads $\nu^{\DPS}_y\propto e^{\beta R_y}\rho_*$ with $\beta=1+2\alpha\kappa(x)>1$. This explains both panels: (1) the mode whose center is further from $x_2=0$ (lower, $x_2\approx-4.9$) has smaller $R_y$ and loses mass to the nearer one (upper, $x_2\approx3.4$), giving the $10^2$ ratio in shares; (2) within each mode, samples shift toward $x_2=0$, giving the skew in (c) relative to (a).

\section{Bias Reduction}\label{sec:reducing_the_variance}
We see in \cref{thm:DPSbias} that the ratio between the density of the DPS sampler and the true posterior can be expressed (approximately) as: 
$\frac{1}{\omega(x)}\approx\E_{X\sim\mathrm{OU}} \left[e^{\int_0^T c_{\DPS}(s, X_s)\,ds}\,\middle|\,X_0=x\right].$

This gives a clear design goal: it is beneficial to design paths that result in small variations in $c_{\DPS}$ over the trajectories. This would correspond to a smaller reaction term in Equation~\eqref{surrogate}, and a smaller bias when we implement the corresponding algorithm. For instance, we can add an extra potential drift $-\vartheta\nabla U$ to the SDE~\eqref{eq:algSDE} that drives trajectories to regions where $c_{\DPS}$ has small oscillations. In terms of the algorithm~\eqref{algorithm}, this amounts to solving
\[
\partial_t\ola\nu_t=\underbrace{\Delta\ola\nu_t-\nabla\cdot(x\,\ola\nu_t)-2\nabla\cdot(\nabla\log\ora\mu_{T-t}^{\DPS}\,\ola\nu_t)}_{\text{Equation~\eqref{eq:DPSreversalPDE}}}+\underbrace{\vartheta\,\nabla\cdot(\nabla U\,\ola\nu_t)}_{\text{extra guidance}},
\]
where the drift intensity $\vartheta\ge 0$ is a hyperparameter. As we see below, such a change in drift can readily be matched with a corresponding change in reaction term that reinterprets the \eqref{surrogate} with the updated drift and a modified reaction term.

\textbf{Interpreting the drift as a reaction term.}\;
Just as diffusion can be recast as a drift involving the score, $\Delta\rho_t=\nabla\cdot(\nabla\log\rho_t\,\rho_t)$, the additional drift generated by $\nabla U$ can be recast as a reaction term through the tautological identity:
\[
c_U \;=\; \frac{\nabla\cdot(\nabla U\,\ola\mu^{\DPS}_t)}{\ola\mu^{\DPS}_t} \;=\; \Delta U+\nabla U\cdot\nabla\log\ora\mu^{\DPS}_{T-t}.
\]
In other words, we can rewrite the \eqref{eq:truereversal} as:
\begin{equation*}
\begin{aligned}
\partial_t\overleftarrow{\mu}^{\DPS}_t={}&\Delta\overleftarrow{\mu}^{\DPS}_t-2\,\nabla\!\cdot\!\big(\nabla\log\overrightarrow{\mu}^{\DPS}_{T-t}\,\overleftarrow{\mu}^{\DPS}_t\big)-\nabla\!\cdot\!\big(x\,\overleftarrow{\mu}^{\DPS}_t\big)\\
&+{\color{brown}\vartheta\,\nabla\!\cdot\!\big(\nabla U\,\overleftarrow{\mu}^{\DPS}_t\big)}-\big(c_{\DPS}(T-t,x)+{\color{brown}\vartheta\,c_U}\big)\,\overleftarrow{\mu}^{\DPS}_t
\end{aligned}
\end{equation*}
In the language of~\cref{thm:DPSbias} (with the DPS SDE augmented by the drift $-\vartheta\nabla U$), this modifies the reaction term to $c_{\mathrm{eff}}=c_{\DPS}+\vartheta\,c_U$. As a consequence, excessively large $\vartheta$ is counterproductive: the reaction term becomes dominated by $\vartheta\,c_U$ and the original bias structure is lost.

\begin{remark}\label{rem:perfect}
There exists in principle a potential $U^*$ satisfying $c_{\DPS}+c_{U^*}=0$ (taking $\vartheta=1$), which would eliminate the bias exactly. Computing $U^*$ directly is exponentially slow; recent work instead approximates it via a variational characterizations, training a non-linear~\citep{guo2026conditional} or linear~\citep{driftlite} neural network for each specific reward.
\end{remark}

\textbf{STSL as a special case.}\; 
Originally motivated by expanding the h-transform $h_t^{\ast}(x)$ \eqref{eqn:h-transform} to second order in $X_t$ to approximate 
the optimal (reaction free) drift, STSL~\citep{STSL} chooses a potential $U$ that drives the trajectory toward low-uncertainty regions of the initial condition $X_0$. Up to constants, the choice is
\[
U(t,x)=\trace\!\left(\Sigma_t(x)\right)=\sum_{i=1}^d\lambda_i(t,x)\;\ge\; 0.
\]
where $\lambda_i(t,x)\ge 0$ are the eigenvalues of $\Sigma_t(x)$. Since the $\lambda_i$ are non-negative, smaller values of $U$ correspond to smaller spread in the dominant eigendirections of $\Sigma_t(x)$, which in turn flattens the spatially varying part of $c_{\DPS}$ in~\eqref{define_c_DPS}. In practice, this yields a better algorithm with reduced output uncertainty~\citep{STSL}. 


\section{Numerical Instabilities of the DPS Algorithm}\label{sec:stability}
As discussed in earlier sections, the algorithm evolution introduces a bias when compared to the surrogate evolution. In this section, we study a different issue with the actual implementation of the DPS algorithm, namely instability of the evolution close to the constraint manifold. In the context of the DPS algorithm proposed in \citet{chung2023diffusion}, we show that the instability unavoidably occurs in the space parallel to the data manifold due to the systematic violation of the forward Euler stability condition. This phenomenon has indeed been observed in practice (\cref{fig:instability}). To mitigate this, a common practice is to ``turn off'' reward guidance close to the data manifold, which we refer to as \textbf{early guidance stopping}, see~\cref{alg:early_stopping}. In words, the stochastic evolution starts with both the score (corresponding to the untilted prior) and reward guidance drift terms until some intermediate time $t_{\mathrm{stop}}\in(0,T)$, after which the diffusion proceeds with only the untilted score. We show that early guidance stopping can be explicitly characterized as an appropriately weighted tilt of the prior.

\textbf{Instability of DPS.}\; We first examine the 
algorithmic implementation of the DPS algorithm, which is conceptualized as an approximation to the solution of the SDE~\eqref{eq:DPS_SDE}. The exact algorithm proposed by the authors of~\citep{chung2023diffusion} is given in \cref{app:dps-algo}. The DPS algorithm progressively denoises over discrete time-steps, with the reward guidance weighted at each time-step through a
guidance %
schedule $\{\zeta_i\}_{i=1}^N$ that is taken to be trajectory-dependent,
\begin{equation}\label{eq:bias_schedule}
  \zeta_i=\frac{\alpha}{\|y-\Ac(\hat{x}_{t_i}(x_i))\|_2},
\end{equation}
where $x_i$ is the iterate at step $i$, corresponding to time $t_i$, and $\hat{x}_{t_i}(x_i)$ is its Tweedie mean,
$\Ac:\R^d\to \R^L$ is a general observation operator, $y\in \R^L$ is the observation, and
$\alpha\in [0.2,1]$ is a problem-dependent hyperparameter. In practice, the guidance schedule has a dramatic performance impact. 

A first observation is that~\eqref{eq:bias_schedule} does not account for the time discretization $\Delta t_i$ of the SDE~\eqref{eq:DPS_SDE}; effectively, this corresponds to multiplying the guidance vector field by a time-dependent factor. In terms of the surrogate~\eqref{surrogate}, this corresponds to the curve
\begin{equation}\label{eq:annealedpath}
t\mapsto \ora\mu^{\eta}_t=\frac{e^{-\alpha\eta_t\|\mathcal{A}(\hat{x}_t(x))-y\|_2}\rho_t}{Z_t},    
\end{equation}
that is, $\ora\mu^{\eta}_t = e^{\alpha\eta_t R_y(\hat{x}_t)}\rho_t/Z_t$ with reward $R_y(x)\coloneqq-\|\mathcal{A}(x)-y\|_2$, and annealing schedule (for the linear noising schedule $\{\beta_i\}_{i=1}^{1000}$ used in the classical DDPM, see \cref{sec:DDPM} for details) given by
\begin{equation}\label{eq:annealingschedule}
\eta_t\approx \,\frac{1.8\cdot 10^{4}}{\sqrt{1 + 6600\, t}}.
\end{equation}
The key takeaway is that the schedule weight is large for a reasonable choice of hyperparameters, with the qualitative implication of {\em strong enforcement of measurement constraints} as we approach the data manifold (small $t$ / low-temperature regime). Indeed notice that this path-dependent guidance schedule yields the target density of:
\begin{equation}\label{eq:target}
\mu^{\mathrm{Target}}_y=\frac{e^{-\alpha\, 1.8\cdot 10^4\|\mathcal{A}(x)-y\|_2}\rho_*}{Z^{\mathrm{Target}}},    
\end{equation}
in which the reward is unsquared and the constraint $\{\mathcal{A}(x)=y\}$ has a large weight.

\begin{figure}[t]
    \centering
    \includegraphics[width=0.9\linewidth]{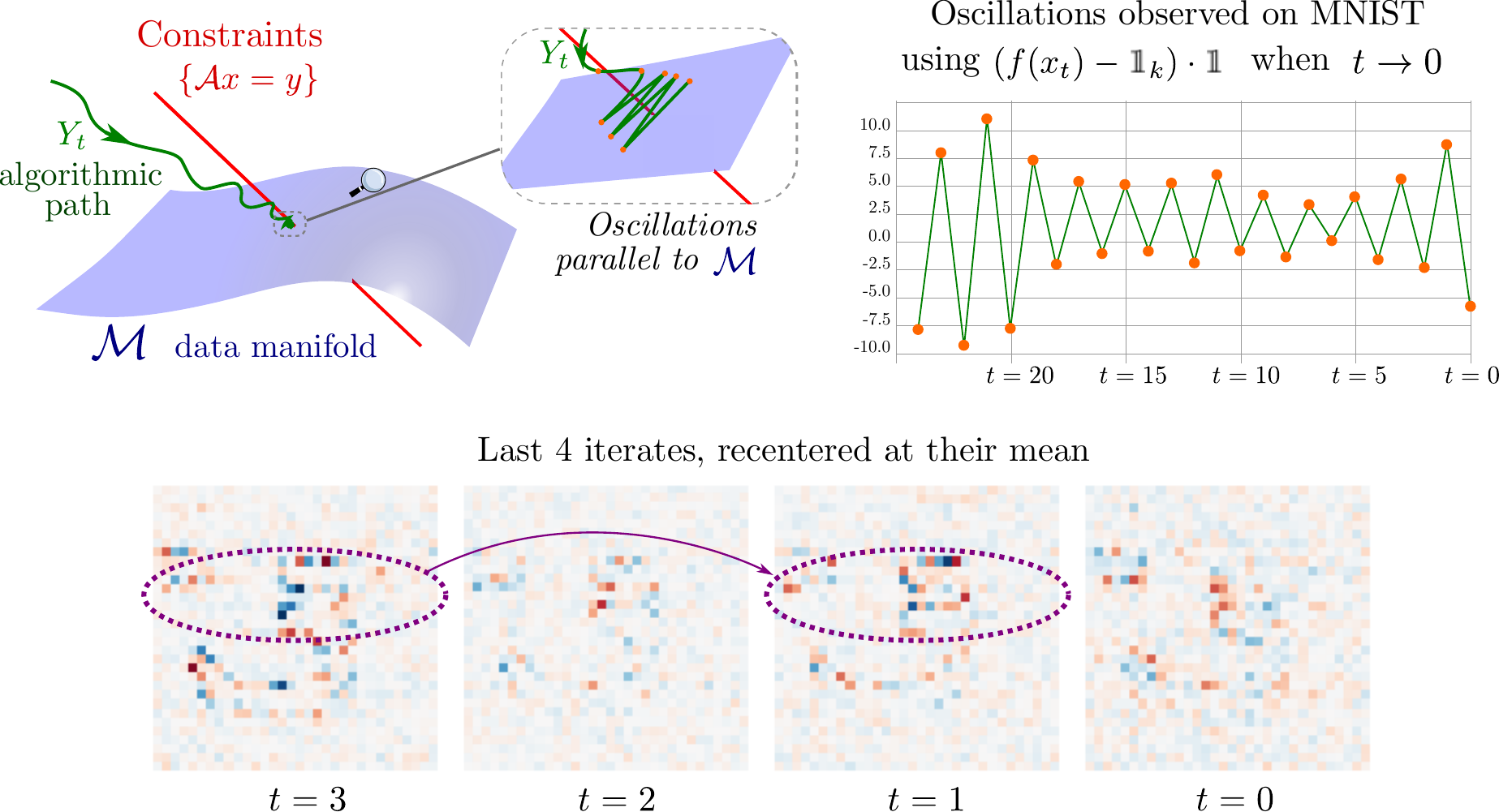}
    \caption{\textbf{DPS oscillates near the data manifold.} \textbf{(Top left)} Schematic: near the manifold, the large effective guidance triggers oscillations. \textbf{(Top right)} Oscillations in an MNIST posterior sampling task, with guidance toward the digit $3$. \textbf{(Bottom)} The last four DPS iterates, centered at their mean: alternate iterates deviate alike, a period-two oscillation. Details in \cref{fig:instability_detailed} and \cref{app:instability}.}
    \label{fig:instability}
\end{figure}

\textbf{Inevitable Oscillations.}\; The unsquared residual in target~\eqref{eq:target} distorts the dynamics in a way that no choice of step size can repair. To see this, consider the one-dimensional example where gradient flow on $|x|$ under forward Euler is: $x_{n+1} = x_n - \Delta t\,\mathrm{sign}(x_n)$. The gradient $\mathrm{sign}(x)$ has unbounded Lipschitz constant at the origin, so any $\Delta t > 0$ produces a limit cycle of amplitude $\sim \Delta t$ around the minimum; this is bounded but non-convergent. The DPS guidance integration is the multidimensional analogue. 


As $Y_t$ approaches the constraint $\{\mathcal{A}(x) = y\}$, the guidance gradient $\nabla_Y\|\mathcal{A}(\hat{x}_{T-t}(Y))-y\|_2 = \nabla\hat{x}_{T-t}(Y)^\top\nabla\mathcal{A}(\hat{x}_{T-t}(Y))^\top(\mathcal{A}(\hat{x}_{T-t}(Y))-y)/\|\mathcal{A}(\hat{x}_{T-t}(Y))-y\|_2$ does not vanish, while the annealing schedule~\eqref{eq:annealingschedule} multiplying the drift cancels the Euler step size exactly (see \cref{sec:instability}). The forward Euler stability criterion is therefore inevitably violated near the constraint, and the iteration enters a limit cycle tangent to the data manifold $\mathcal{M}$, of amplitude $\sim \alpha\,\sigma_{\max}(\nabla\mathcal{A}\,P_{T_p\mathcal{M}})$. Here $P_{T_p\mathcal{M}}$ is the orthogonal projection onto the tangent space $T_p\mathcal{M}$ of $\mathcal{M}$ at the nearest point $p\in\mathcal{M}$, the small-noise limit of $\nabla\hat{x}_t$ (\cref{sec:zeronoiselimit}), and $\sigma_{\max}$ denotes the largest singular value.
%
The advantage of the unsquared residual $\|\cdot\|_2$ over $\|\cdot\|_2^2$ in the reward is that the iterates remain \emph{semi-stable} in the sense of Lyapunov: they settle into a limit cycle at distance at most $\alpha\,\|\nabla\mathcal{A}\|_{\mathrm{op}}$ from the constraint manifold. By contrast, when the forward Euler stability criterion is violated for $\|\cdot\|_2^2$, the oscillations diverge.%

\textbf{Early Guidance Stopping.}\; To avoid these numerical instabilities, practitioners apply early guidance stopping (\cref{alg:early_stopping}), terminating the guidance at some intermediate time $t_{\mathrm{stop}}\in (0,T)$. 
Combined with~\cref{thm:DPSbias}, this recovers the output of DPS with early guidance stopping.
\begin{restatable}[Early Guidance Stopping]{theorem}{earlythm}\label{thm:early}
If guidance is stopped at time $t_{\mathrm{stop}}=T-t_*$, the output of the DPS algorithm, up to time-discretization error, is given by
\[
\nu^{\DPS,t_*}_y(x)=\frac{\E_{X\sim\mathrm{OU}}\left[\frac{1}{\omega_{t_*}}(X_{t_*})\, e^{\alpha\eta_{t_*}R_y(\hat{x}_{t_*}(X_{t_*}))}\,\bigg|\,X_0=x\right]}{Z_{t_*}}\rho_*(x),
\]
where
\begin{equation}\label{eq:early_weight}
    \frac{1}{\omega_{t_*}}(x) \coloneqq \E_{X\sim\mathrm{OU}}\left[\frac{\gamma(X_{T-t_*})}{\overrightarrow{\mu}^{\eta}_T(X_{T-t_*})}\exp\left(\int_0^{T-t_*} c^*_{\DPS}(t_*+s,X_s)\,ds\right) \,\bigg\vert\, X_0 = x\right],
\end{equation}
with
\[
\begin{aligned}
c^*_{\DPS}(t,x)={}& -\frac{\alpha\eta_t\,\trace\!\big(\Sigma_t(x)(D^2R_y)(\hat{x}_t(x))\Sigma_t(x)\big)+\alpha^2\eta_t^2\,\big\|\Sigma_t(x)\nabla R_y(\hat{x}_t(x))\big\|^2}{(e^t-e^{-t})^2}\\
&+\alpha\frac{d\eta_t}{dt}\,R_y(\hat{x}_t(x))-\frac{d}{dt}\log Z_t,
\end{aligned}
\]
that is, \eqref{define_c_DPS} with $R_y$ replaced by $\alpha\eta_t R_y$, plus a term from the time dependence of $\eta_t$.
Here $\alpha$, $R_y$, the \textbf{annealing schedule} $\eta_t$ and $Z_t$ are as in \eqref{eq:bias_schedule}--\eqref{eq:annealingschedule}.
\end{restatable}
See for instance \cite[Proposition 2.6]{huang2026guide} for the effect of early guidance stopping in the simpler linear-quadratic case.

\section*{Acknowledgments}
The research of AP and SS has been partially supported by NSF Grants 2019844, 2505865 and 2112471, and the UT Austin Machine Learning Lab. The research of MGD was partially supported by NSF-DMS-2205937.

\newpage

\section*{Reproducibility Statement}
All theoretical results are stated with their assumptions, and complete proofs are given in the appendix. Our 2-dimensional Gaussian mixture experiments use standard methods run in their standard form, including DDPM ancestral sampling with exact scores and jacobians, and the calculation of $\frac{1}{\omega}$ for Figure 2 is 
a standard finite-difference scheme with full configuration given in \cref{app:sampling_figure}. The MNIST experiments run the standard DPS algorithm of \citet{chung2023diffusion} as specified in \cref{app:dps-algo}, and their configurations are summarized in \cref{app:instability}.

\section*{AI Use Statement}
In this work, we used generative AI tools to assist in the writing and checking of proofs (verifying derivations, identifying and correcting errors in intermediate formulas, and drafting revised versions of appendix proofs), to provide feedback on the research and its presentation, and to implement methods used in our numerical experiments. We have not used generative AI tools to formulate the main mathematical claims or conceptual framework, to propose hypotheses or interpret results. Generating synthetic data sets, translation and qualitative or thematic data analysis are not applicable to this work. Additionally, we used generative AI tools to edit the paper for readability, draft parts of the text, to create and edit software code, to search for and identify relevant literature, and to check and format references. We have reviewed all AI-assisted work: every line has been examined and revised by the authors and we take responsibility for the final content of this work, including text, claims and artifacts produced with the aid of generative AI.

\bibliographystyle{iclr2027_conference}
\bibliography{refs}

\newpage
\appendix

\section{Experimental Details for \texorpdfstring{\cref{fig:sampling_figure}}{Figure 1}}\label{app:sampling_figure}

\textbf{Results.} Mode shares are reported as upper/lower. \textbf{(a)}~$N=10^5$ exact posterior draws split $50.1\%/49.9\%$. \textbf{(b)}~The weight $1/\omega=d\nu^{\DPS}_y/d\mu_y$, on a log scale. Warm colors predict over-sampling by DPS, cool colors predict under-sampling, and the white contour is $1/\omega=1$. Integrating $1/\omega$ against $\mu_y$ predicts $99.05\%/0.95\%$, odds of $104\!:\!1$. As a check on the solve, $\int(1/\omega)\,d\mu_y=1$ evaluates to $1.087$ on the posterior support. \textbf{(c)}~$N=10^5$ DPS samples split $99.07\%/0.93\%$, odds of $106\!:\!1$, matching (b) in both direction and magnitude. Panels (a) and (c) show a random subset of $4000$ samples, while all shares use the full $10^5$.

\textbf{Setup.} The prior is a two-component Gaussian mixture on $\R^2$ with means $(0,\pm5.5)$ and covariances $\operatorname{diag}(0.3,1.5)$ (upper) and $\operatorname{diag}(1.5,0.3)$ (lower). The measurement is $\mathcal{A}(x)=e_2^\top x$ with noise precision $\alpha=0.4$ and observation $y=0$, so $R_y(x)=-\tfrac{\alpha}{2}x_2^2$. The prior weights $0.191/0.809$ are chosen so that the two components of the posterior, which is again a Gaussian mixture in closed form, carry equal mass. Dashed rings are the $1$, $2$, $3\sigma$ ellipses of these posterior components. DPS runs \eqref{eq:DPS_SDE} by DDPM ancestral sampling on the linear schedule of \cref{sec:DDPM} ($1000$ steps), using the mixture's exact score, Tweedie mean and Jacobian, so no network is trained. Mode shares are sample averages of the component responsibilities $r_k(x)=p_k\mathcal{N}(x;m_k,\Sigma_k)/\rho_*(x)$, which the tilt $e^{R_y}$ leaves unchanged.

\textbf{Computing $1/\omega$.} By \cref{lem:backward_ratio_fk}, the ratio $\psi_t\coloneqq\ola\nu_t/\ola\mu_t$ of the DPS law to the surrogate path solves
\begin{equation}\label{eq:psi_pde_fig}
\partial_t\psi_t=L\psi_t+c_{\DPS}(T-t,\cdot)\,\psi_t,\qquad \psi_0=\gamma/\ora\mu^{\DPS}_T,\qquad Lf=\Delta f-x\cdot\nabla f,
\end{equation}
and its terminal value is $1/\omega$: at $t=T$, $\ola\nu_T=\nu^{\DPS}_y$ and $\ola\mu_T=\mu_y$, so $\psi_T=\nu^{\DPS}_y/\mu_y=1/\omega$ by \eqref{eq:fk_DPS}. We solve \eqref{eq:psi_pde_fig} numerically rather than estimating \eqref{eq:weight_DPS_forward} by Monte Carlo. The term $-\tfrac{d}{ds}\log Z_s$ in $c_{\DPS}$ is constant in space and only rescales the solution. Using $\rho_T\approx\gamma$, we therefore drop it and start instead from $\psi_0=Z_0e^{-R_y(\hat{x}_T)}$, which gives the same $\psi_T$. The remaining reaction term, $Z_0$, and $\hat{x}_T$ are all available in closed form for a Gaussian mixture. The horizon matches the sampler, $T\approx5.06$.

\textbf{Solver.} We discretize space with a uniform $801\times801$ grid covering the plotted window plus a margin of $4$ on each side, and impose reflecting (zero-flux) boundary conditions at its edges. The margin makes the boundary irrelevant: OU paths started in the window rarely reach it. Time steps are the $999$ OU time steps of the DDPM sampler. They are finest near the data, where the reaction term changes fastest, and they put the solver and the sampler on the same clock. Each step treats the two parts of \eqref{eq:psi_pde_fig} separately. The reaction part is a pointwise multiplication by an exponential, which is exact. The diffusion part is a Crank--Nicolson step, taken one coordinate at a time, so each step reduces to tridiagonal linear solves. Because \eqref{eq:psi_pde_fig} is linear, we rescale $\psi$ to unit maximum after every step and track the scale separately, which avoids numerical overflow. The mode shares and the normalization check in (b) come from integrating $1/\omega$ against $\mu_y$ on the grid, restricted to the region where the posterior density exceeds $10^{-7}$ of its maximum.

\section{Feynman--Kac representations of the Radon--Nikodym derivative}\label{app:FK_background}
We consider two densities $\pi_t,\pi'_t$ evolving under a pair of equations of the form \eqref{eq:fp-pair} below. The surrogate path $\ola\mu_t$ \eqref{surrogate} and the algorithm path $\ola\nu_t$ \eqref{algorithm} are such an example, with $\pi = \ola\mu$, $\pi' = \ola\nu$, a common transport field, and $f'\equiv 0$.
\begin{lemma}[Feynman--Kac for Density Ratio]\label{lem:FK_ratio}
Let $T > 0$ and $\alpha \in (0,1)$. Consider two initial measures with Lebesgue densities $\pi_0(x)\,dx$ and $\pi'_0(x)\,dx$, where $\pi_0,\pi'_0 \in C^{2+\alpha}_{loc}(\bbrd)$ and $\pi_0(x),\pi'_0(x) > 0$ for all $x \in \bbrd$. Suppose the Radon--Nikodym derivative $g_0 \coloneqq d\pi'_0/d\pi_0 = \pi'_0/\pi_0 \in C^2(\bbrd)$ satisfies the growth bound
$\abs{g_0(x)} + \|\grad g_0(x)\| + \|D^2 g_0(x)\| \leq C\, e^{\lambda \|x\|^2}$ for some $C,\lambda$. Suppose also that $v,v',f,f'$ satisfy
    \begin{align*}
    & v, v' \in C^{\alpha/2,\,1+\alpha}_{\mathrm{loc}}\big((0,T) \times \bbrd\big), \\ 
    & f, f' \in C^{\alpha/2,\,\alpha}_{\mathrm{loc}}\big((0,T) \times \bbrd\big),
    \end{align*}
and the growth conditions, for some fixed $K_1,K_2$ (independent of $t$),
\begin{align*} 
& \|v_t(x)\| + \|v'_t(x)\| \leq K_1(1 + \|x\|), \qquad (t,x) \in (0,T) \times \bbrd,\\ 
& \abs{f_t(x)} + \abs{f'_t(x)} \leq K_2(1+\|x\|^2), \qquad (t,x) \in (0,T) \times \bbrd.
\end{align*} 
Then we have the following:
\begin{enumerate}
    \item[(i)] There exist unique classical solutions $\pi, \pi' \in C^{1,2}\big((0,T) \times \bbrd\big)\cap C\big([0,T)\times\bbrd\big)$ 
    to 
    \begin{equation}\label{eq:fp-pair}
    \begin{cases}
        \partial_t \pi_t = -\nabla \cdot (v_t\, \pi_t) + \Delta \pi_t + f_t\, \pi_t, 
            & \pi|_{t=0} = \pi_0, \\[4pt]
        \partial_t \pi'_t = -\nabla \cdot (v'_t\, \pi'_t) + \Delta \pi'_t + f'_t\, \pi'_t, 
            & \pi'|_{t=0} = \pi'_0,
    \end{cases}
    \end{equation}
    and $\pi_t, \pi'_t > 0$ for all $t \in (0,T)$.
    \item[(ii)] The ratio $g_t(x) \coloneqq \pi'_t(x)/\pi_t(x)$ belongs to 
    $C^{1,2}\big((0,T) \times \bbrd\big)\cap C\big([0,T)\times\bbrd\big)$ and is a classical solution to
    \begin{equation}\label{eq:ratio-pde-app}
        \partial_t g_t = \Delta g_t + b_t \cdot \nabla g_t - c_t\, g_t,
        \qquad g|_{t=0} = g_0,
    \end{equation}
    with
    \begin{equation}\label{eq:fk-coefficients}
        b_t \coloneqq 2\,\nabla \log \pi_t - v'_t, 
        \qquad c_t \coloneqq f_t - f'_t - \nabla \cdot \big(v_t - v'_t\big) - \big(v_t - v'_t\big)\cdot\nabla\log\pi_t .
    \end{equation}
    \item[(iii)] Fix $t\in(0,T)$ and suppose in addition that
    \begin{enumerate}
        \item[(a)] $b$ has linear growth, $\|b_r(x)\| \le K_3(1+\|x\|)$ for $(r,x)\in(0,t]\times\bbrd$, and is locally Lipschitz in $x$ uniformly in $r\in(0,t]$;
        \item[(b)] $c$ is bounded below, $c_r(x) \ge -K_4$ for $(r,x)\in(0,t]\times\bbrd$;
        \item[(c)] $|g_r(x)| \le C'\, e^{\lambda'\|x\|^2}$ for $(r,x)\in[0,t]\times\bbrd$, where $\lambda'>0$ is small enough that $\E\big[\sup_{s\le t} e^{\lambda'\|A_s\|^2}\bigm| A_0 = x\big] < \infty$ for the process $A$ below.
    \end{enumerate}
    Then the SDE
    \begin{equation}\label{eq:fk-sde-app}
        dA_s = b_{t-s}(A_s)\, ds + \sqrt{2}\, dB_s, \qquad A_0 = x, \qquad s\in[0,t],
    \end{equation}
    where $B$ is a standard $d$-dimensional Brownian motion, has a unique strong solution, and the ratio admits the Feynman--Kac representation
    \begin{equation}\label{eq:fk-formula}
        g_t(x) = \E\!\left[ g_0(A_t)\, 
            \exp\!\left( -\int_0^t c_{t-s}(A_s)\, ds \right) \bigg|\, A_0 = x \right].
    \end{equation}
    In particular, $g$ is the unique classical solution of \eqref{eq:ratio-pde-app} satisfying (c).
\end{enumerate}
\end{lemma}
Condition (c) holds for all sufficiently small $\lambda'$ (depending on $t$ and $K_3$) under (a), since then $\sup_{s\le t}\|A_s\| \le e^{K_3 t}\big(\|x\| + K_3 t + \sqrt2\sup_{s\le t}\|B_s\|\big)$ by Gr\"onwall's inequality and $\sup_{s\le t}\|B_s\|$ has Gaussian tails.

\begin{remark}[The two uses in the main text]\label{rem:FK_uses}
In \cref{thm:DPSbias}(ii) we take $\pi = \ola\mu^{\DPS}$, $\pi' = \ola\nu^{\DPS}$, with common transport $v_t = v'_t = x + 2\nabla\log\ora\mu^{\DPS}_{T-t} = x + 2\nabla\log\pi_t$, reaction $f_t = -c_{\DPS}(T-t,\cdot)$ and $f'_t = 0$. Then \eqref{eq:fk-coefficients} gives
\[
b_t(x) = -x, \qquad c_t(x) = -c_{\DPS}(T-t,x),
\]
so $A$ in \eqref{eq:fk-sde-app} is the OU process \eqref{eq:SDE}, and \eqref{eq:fk-formula} becomes
\[
\frac{\ola\nu^{\DPS}_t}{\ola\mu^{\DPS}_t}(x) = \E_{X\sim\mathrm{OU}}\!\left[\frac{\gamma(X_t)}{\ora\mu^{\DPS}_T(X_t)}\exp\!\left(\int_0^t c_{\DPS}(T-t+s, X_s)\,ds\right)\bigg|\,X_0 = x\right],
\]
which is \cref{lem:backward_ratio_fk}. In the proof of \cref{thm:early}, on $[t_{\mathrm{stop}},T)$ we take $\pi = \ola\rho$ and $\pi' = \ola\nu$ with common transport and no reaction, so again $b_t(x) = -x$ and $c_t\equiv 0$. In both cases $b$ is time-independent, so the time-reversal $b_{t-s}$ in \eqref{eq:fk-sde-app} has no effect there; it matters only for general $v'$.
\end{remark}

\begin{proof}[Proof of \cref{lem:FK_ratio}]
That the assumptions on the coefficients and measures $\pi'_0,\pi_0$ imply (i) is a classical result \cite[Ch.~IV, Thm.~5.1]{lsu}.

\emph{(ii).} By (i), $g_t = \pi'_t/\pi_t$ is well defined, positive, and inherits the regularity of $\pi,\pi'$. Substitute $\pi'_t = g_t\pi_t$ into the second equation of \eqref{eq:fp-pair}, using
\[
\partial_t(g_t\pi_t) = \pi_t\,\partial_t g_t + g_t\,\partial_t\pi_t,\qquad
\laplacian(g_t\pi_t) = \pi_t\laplacian g_t + 2\grad g_t\cdot\grad\pi_t + g_t\laplacian\pi_t,
\]
\[
\nabla\cdot(v'_t g_t\pi_t) = \pi_t\, v'_t\cdot\grad g_t + g_t\,\nabla\cdot(v'_t\pi_t).
\]
Subtracting $g_t$ times the first equation of \eqref{eq:fp-pair} gives
\[
\pi_t\,\partial_t g_t = \pi_t\laplacian g_t + \big(2\grad\pi_t - v'_t\pi_t\big)\cdot\grad g_t + \Big(\nabla\cdot\big((v_t - v'_t)\pi_t\big) + (f'_t - f_t)\,\pi_t\Big)\, g_t .
\]
Dividing by $\pi_t > 0$ and expanding $\frac{1}{\pi_t}\nabla\cdot\big((v_t-v'_t)\pi_t\big) = \nabla\cdot(v_t - v'_t) + (v_t - v'_t)\cdot\grad\log\pi_t$ yields \eqref{eq:ratio-pde-app} with $b_t, c_t$ as in \eqref{eq:fk-coefficients}.

\emph{(iii).} By (a), the drift $s\mapsto b_{t-s}$ is locally Lipschitz with linear growth on $[0,t)\times\bbrd$, so \eqref{eq:fk-sde-app} has a unique, non-exploding strong solution \citep{karatzas}. For $\varepsilon\in(0,t)$ define, for $s\in[0,t-\varepsilon]$, the process
\[
    M_s \coloneqq g_{t-s}(A_s)\, \exp\!\left( -\int_0^s c_{t-r}(A_r)\, dr \right),
\]
which depends on the whole trajectory. Since $t-s\ge\varepsilon>0$, $g$ is $C^{1,2}$ along the path, and It\^o's formula gives
\[
dM_s = e^{-\int_0^s c_{t-r}(A_r)\,dr}\Big[\big(-\partial_t g_{t-s} + \laplacian g_{t-s} + b_{t-s}\cdot\grad g_{t-s} - c_{t-s}\,g_{t-s}\big)(A_s)\,ds + \sqrt{2}\,\grad g_{t-s}(A_s)\cdot dB_s\Big].
\]
The $ds$ term vanishes by \eqref{eq:ratio-pde-app} evaluated at time $t-s$; this is why the characteristics carry the time-reversed drift $b_{t-s}$. Hence
\[
    dM_s = \sqrt{2}\, e^{-\int_0^s c_{t-r}(A_r)\, dr}\, \grad g_{t-s}(A_s) \cdot dB_s,
\]
so $M$ is a local martingale on $[0,t-\varepsilon]$. By (b) and (c),
\[
\sup_{s\le t-\varepsilon}|M_s| \;\le\; C' e^{K_4 t}\,\sup_{s\le t} e^{\lambda'\|A_s\|^2},
\]
which is integrable and independent of $\varepsilon$. A local martingale dominated by an integrable random variable is a true martingale, so $\E[M_{t-\varepsilon}\mid A_0 = x] = M_0 = g_t(x)$ (see \cite[Theorem 5.7.6]{karatzas} for the standard presentation). As $\varepsilon\downarrow 0$, $M_{t-\varepsilon}\to g_0(A_t)\exp\big(-\int_0^t c_{t-r}(A_r)\,dr\big)$ almost surely, by continuity of $g$ on $[0,T)\times\bbrd$ and of the paths of $A$. Dominated convergence, with the same bound, gives \eqref{eq:fk-formula}. The same argument applies to any classical solution of \eqref{eq:ratio-pde-app} satisfying (c), which gives uniqueness.
\end{proof}

\section{Bounds and Identities on the OU process}
\subsection{The effective sample backward path}\label{sec:relentropy}
The starting point for sampling from a prior in score-based generative models is the forward path $t\mapsto\rho_t$ that interpolates between the prior $\rho_0=\rho_*$ and the standard Gaussian $\gamma=\mathcal{N}(0,I)$. This is obtained by solving the OU process $\rho_t=\mathrm{Law}(X_t)$, where $\{X_t\}_{t\ge 0}$ satisfies the SDE
\begin{equation*}
    \begin{cases}
        dX_t=-X_t\,dt+\sqrt{2}\,dB_t,\\
        X_0\sim\rho_*,
    \end{cases}
\end{equation*}
or, equivalently, the Fokker--Planck equation
\begin{equation}\label{eq:FP}
\begin{cases}
\partial_t\rho_t=\Delta\rho_t+\nabla\cdot(x\,\rho_t),\\
\rho_0=\rho_*.
\end{cases}
\end{equation}
By the log-Sobolev inequality for $\gamma$ and the relative-entropy decay along the OU flow, we have the exponentially decaying bound for $t>1$,
\begin{equation}\label{eq:expdecay}
\mathcal{H}(\rho_t\,|\,\gamma)\le C\, e^{-t},
\end{equation}
where $C$ is a universal constant independent of dimension.

To obtain approximate samples from $\rho_*$, we fix a time horizon $T$ and reverse the path: set $\overleftarrow{\rho}_t=\rho_{T-t}$. The reverse path satisfies
\begin{equation}\label{eq:Back}
\begin{cases}
\partial_t\overleftarrow{\rho}_t=\Delta\overleftarrow{\rho}_t-2\nabla\cdot(\nabla\log\rho_{T-t}\,\overleftarrow\rho_t)-\nabla\cdot(x\,\ola{\rho}_t),\\
\overleftarrow{\rho}_0=\rho_T,
\end{cases}
\end{equation}
where we have used the diffusion-to-drift identity
\begin{equation*}
\Delta\overleftarrow{\rho}_t=-\Delta\overleftarrow{\rho}_t+2\nabla\cdot(\nabla\log\rho_{T-t}\,\overleftarrow\rho_t).
\end{equation*}
In practice, the initial condition is replaced by a standard Gaussian $\ola\rho^{\mathrm{eff}}_0=\gamma$. The computationally intensive part of this strategy is obtaining a good approximation $s_\theta(x,t)\approx\nabla\log\rho_{t}(x)$ from samples; see~\cref{sec:background}. The effective samples are then obtained by approximating the solution of
\begin{equation}\label{eq:Backeff}
\begin{cases}
\partial_t\overleftarrow{\rho}^{\mathrm{eff}}_t=\Delta\overleftarrow{\rho}^{\mathrm{eff}}_t-2\nabla\cdot(s_\theta(\cdot,T-t)\,\overleftarrow\rho^{\mathrm{eff}}_t)-\nabla\cdot(x\,\ola{\rho}^{\mathrm{eff}}_t),\\
\overleftarrow{\rho}^{\mathrm{eff}}_0=\gamma,
\end{cases}
\end{equation}
which arises as the law of
\begin{equation*}
    \begin{cases}
        d\tilde{X}_t=\tilde{X}_t\,dt+2\,s_\theta(\tilde{X}_t,T-t)\,dt+\sqrt{2}\,dB_t,\\
        \tilde{X}_0\sim\gamma.
    \end{cases}
\end{equation*}
Differentiating the relative entropy $\tfrac{d}{dt}\mathcal{H}(\ola\rho^{\mathrm{eff}}_t\,|\,\ola\rho_t)$ along the flow yields the relative-entropy bound between the effective samples $\tilde{X}_T\sim\ola\rho_T^{\mathrm{eff}}$ and the true distribution $X_0\sim\rho_*$,
\begin{equation}\label{eq:relativenetropy}
    \mathcal{H}(\ola\rho_T^{\mathrm{eff}}\,|\,\rho_*)\;\le\;  \underbrace{\mathcal{H}(\rho_T\,|\,\gamma)}_{\le\, C e^{-T}}\;+\;\underbrace{\int_0^T \int_{\R^d}\|\nabla\log\rho_t(x)-s_\theta(x,t)\|^2\,\rho_t(x)\,dx\,dt}_{\text{Score-approximation error}}.
\end{equation}
In what follows, we assume the score-approximation error is negligible, so the samples obtained by denoising are for practical purposes indistinguishable from $\rho_*$.

\subsection{Tweedie's identities}\label{sec:Tweedies}

A widely used heuristic is to estimate $X_0$ from a noisy observation $X_t$ via the conditional expectation
\begin{equation}\label{eq:tweedie_def}
\hat{x}_t(x) \;\coloneqq\; \E[X_0 \mid X_t = x],
\end{equation}
which, under the negligible-score-approximation assumption of~\cref{sec:relentropy}, also equals $\E[\tilde{X}_T \mid \tilde{X}_{T-t} = x]$ for the effective backward process. Tweedie's formula expresses~\eqref{eq:tweedie_def} in closed form via the score:
\begin{equation}\label{eq:tweedie}
\hat{x}_t(x) \;=\; e^t\,x \;+\; (e^t - e^{-t})\,\nabla\log\rho_t(x).
\end{equation}

\paragraph{Derivation.} The OU semigroup admits the Gaussian kernel
\begin{equation}\label{eq:ou_kernel}
\rho_{t\mid 0}(x \mid x_0) \;=\; \frac{1}{(2\pi(1-e^{-2t}))^{d/2}} \exp\!\left(-\frac{\|x - e^{-t} x_0\|^2}{2(1-e^{-2t})}\right),
\end{equation}
so $\nabla_x \log \rho_{t\mid 0}(x \mid x_0) = -(x - e^{-t} x_0)/(1-e^{-2t})$. Differentiating $\rho_t(x) = \int \rho_{t\mid 0}(x\mid x_0)\,\rho_*(x_0)\,dx_0$ in $x$ and dividing by $\rho_t(x)$,
\[
\nabla\log\rho_t(x) \;=\; -\frac{x - e^{-t}\,\hat{x}_t(x)}{1-e^{-2t}}.
\]
Solving for $\hat{x}_t(x)$ and using $e^t(1-e^{-2t}) = e^t - e^{-t}$ yields~\eqref{eq:tweedie}. \hfill$\square$

\paragraph{Second-order identity.} Differentiating~\eqref{eq:tweedie} in $x$,
\begin{equation}\label{eq:tweedie2}
\nabla\hat{x}_t(x) \;=\; e^t\,I \;+\; (e^t - e^{-t})\,D^2\log\rho_t(x).
\end{equation}
The right-hand side admits a probabilistic interpretation as a rescaled conditional covariance:
\begin{equation}\label{eq:tweedie_cov}
\Sigma_t(x) \;\coloneqq\; \mathrm{Cov}(X_0 \mid X_t = x) \;=\; (e^t - e^{-t})\,\nabla\hat{x}_t(x).
\end{equation}

\paragraph{Proof of~\eqref{eq:tweedie_cov}.} Differentiating $\rho_t$ twice via~\eqref{eq:ou_kernel} and subtracting $(\nabla\log\rho_t)(\nabla\log\rho_t)^\top$ to convert from $D^2\rho_t/\rho_t$ to $D^2\log\rho_t$,
\begin{equation}\label{eq:hess_log_rho}
D^2\log\rho_t(x) \;=\; -\frac{I}{1-e^{-2t}} \;+\; \frac{e^{-2t}}{(1-e^{-2t})^2}\,\Sigma_t(x).
\end{equation}
Substituting into~\eqref{eq:tweedie2}, the identity $(e^t - e^{-t})/(1-e^{-2t}) = e^t$ cancels the $e^t I$ contribution, and the identity $(e^t - e^{-t})\,e^{-2t}/(1-e^{-2t})^2 = e^{-t}/(1-e^{-2t}) = 1/(e^t - e^{-t})$ collapses the covariance term, giving
\[
\nabla\hat{x}_t(x) \;=\; \frac{\Sigma_t(x)}{e^t - e^{-t}},
\]
which rearranges to~\eqref{eq:tweedie_cov}. \hfill$\square$

\subsubsection{The zero-noise limit}\label{sec:zeronoiselimit}

Throughout this subsection we assume $\rho_*$ is supported on a smooth, compact, $k$-dimensional submanifold $\mathcal{M} \subset \R^d$ of positive reach $\tau_\mathcal{M} > 0$, with a smooth positive density with respect to the volume measure on $\mathcal{M}$. The orthogonal projection $P_\mathcal{M}: x \mapsto \arg\min_{x_0 \in \mathcal{M}} \|x - x_0\|$ is then well-defined and smooth on the tubular neighborhood $\mathcal{N}_{\tau_\mathcal{M}} \coloneqq \{x \in \R^d : \mathrm{dist}(x, \mathcal{M}) < \tau_\mathcal{M}\}$, and the case $x\in\mathcal{M}$ corresponds to $\mathrm{dist}(x,\mathcal{M})=0$. We compute the small-$t$ behavior of $\hat{x}_t$ and $\nabla\hat{x}_t$ via Laplace's method on $\mathcal{M}$; standard references include~\citet{varadhan1966asymptotic, dembo2009large}.

\paragraph{Setup.} As $t \to 0^+$, $1 - e^{-2t} = 2t + O(t^2)$ and $e^{-t} = 1 + O(t)$, so the OU kernel~\eqref{eq:ou_kernel} concentrates as
\[
\rho_{t\mid 0}(x \mid x_0) \;=\; \frac{1+O(t)}{(4\pi t)^{d/2}} \exp\!\left(-\frac{\|x - x_0\|^2}{4t} + O(1)\right) \qquad \text{as } t \to 0^+.
\]
The integrals defining $\hat{x}_t(x)$ are of Laplace form on $\mathcal{M}$ with phase $\Phi(x_0) = \|x - x_0\|^2$ at temperature $4t$.

\paragraph{Local geometry.} Fix $x \in \mathcal{N}_{\tau_\mathcal{M}}$, set $p \coloneqq P_\mathcal{M}(x)$, and let $P_{T_p\mathcal{M}}$ be the orthogonal projection onto the tangent space $T_p\mathcal{M}$. Parametrize $\mathcal{M}$ near $p$ by tangent vectors,
\[
\varphi(v) \;=\; p + v + O(\|v\|^2), \qquad v \in T_p\mathcal{M},
\]
where the $O(\|v\|^2)$ correction lies in the normal space $N_p\mathcal{M} \coloneqq (T_p\mathcal{M})^\perp$ and is bounded by the second fundamental form. Since $x - p \in N_p\mathcal{M}$ is orthogonal to $v\in T_p\mathcal{M}$,
\begin{equation}\label{eq:phase_expansion}
\|x - \varphi(v)\|^2 \;=\; \|x - p\|^2 + \|v\|^2 + O(\|v\|^3),
\end{equation}
so the phase $\Phi(x_0) = \|x - x_0\|^2$ is locally quadratic in tangent coordinates with Hessian $2\,I_{T_p\mathcal{M}}$ at the minimizer $p$.

\paragraph{Limit of $\hat{x}_t$.} On $\mathcal{N}_{\tau_\mathcal{M}}$, the phase $\Phi|_\mathcal{M}$ has unique global minimizer $p$ with $\Phi(p) = \|x-p\|^2$. Laplace's method on $\mathcal{M}$, applied with the local expansion~\eqref{eq:phase_expansion}, gives the asymptotic
\[
\int_\mathcal{M} f(x_0)\,e^{-\Phi(x_0)/(4t)}\,\rho_*(x_0)\,d\mathrm{vol}(x_0) \;=\; (4\pi t)^{k/2}\,e^{-\Phi(p)/(4t)}\,\rho_*(p)\,\big(f(p) + O(t)\big)
\]
for any $C^1$ function $f$ on $\mathcal{M}$. Recall that the Tweedie estimate is the ratio
\[
\hat{x}_t(x) \;=\; \frac{\int_\mathcal{M} x_0\,e^{-\Phi(x_0)/(4t)}\,\rho_*(x_0)\,d\mathrm{vol}(x_0)}{\int_\mathcal{M} e^{-\Phi(x_0)/(4t)}\,\rho_*(x_0)\,d\mathrm{vol}(x_0)}.
\]
Applying the asymptotic to numerator and denominator, the common prefactor $(4\pi t)^{k/2}\,e^{-\Phi(p)/(4t)}\,\rho_*(p)$ cancels, leaving $\hat{x}_t(x) = p + O(t)$, where the $O(t)$ remainder collects the next-order Laplace corrections. The $O(\sqrt{t})$ contributions from tangential fluctuations vanish by Gaussian symmetry on $T_p\mathcal{M}$, since the linear function $v\mapsto v$ has zero mean under a centered Gaussian. Hence
\begin{equation}\label{eq:hatxlimit}
\hat{x}_t(x) \;=\; P_\mathcal{M}(x) + O(t), \qquad \lim_{t \to 0^+} \hat{x}_t(x) \;=\; P_\mathcal{M}(x).
\end{equation}

\paragraph{Limit of $\nabla\hat{x}_t$.} By~\eqref{eq:tweedie_cov} and $e^t - e^{-t} = 2t + O(t^3)$,
\begin{equation}\label{eq:tweedie_cov_smallt}
\nabla\hat{x}_t(x) \;=\; \frac{\Sigma_t(x)}{e^t - e^{-t}} \;=\; \frac{\Sigma_t(x)}{2t}\big(1+O(t^2)\big),
\end{equation}
so it suffices to compute $\Sigma_t(x)$ to leading order. By~\eqref{eq:phase_expansion}, the conditional law $\rho_{0\mid t}$ is asymptotically Gaussian on $T_p\mathcal{M}$ with covariance $2t\,I_{T_p\mathcal{M}}$, and the normal component of $X_0 - p$ is of order $\|v\|^2 = O(t)$ and contributes only at order $t^2$. Hence
\begin{equation}\label{eq:Sigma_smallt}
\Sigma_t(x) \;=\; 2t\,P_{T_p\mathcal{M}} + O(t^2),
\end{equation}
and substituting into~\eqref{eq:tweedie_cov_smallt} yields
\begin{equation}\label{eq:gradhatxlimit}
\lim_{t \to 0^+} \nabla\hat{x}_t(x) \;=\; P_{T_{P_\mathcal{M}(x)}\mathcal{M}}.
\end{equation}

\paragraph{Consequence.} In the small-noise regime relevant near the constraint set, $\hat{x}_t$ acts as the orthogonal projection onto the data manifold and $\nabla\hat{x}_t$ acts as the projection onto the corresponding tangent space. This is the geometric structure that drives the manifold-tangent oscillations of the DPS guidance analyzed in~\cref{sec:stability}.

\section{The DPS Algorithm~\texorpdfstring{\citep{chung2023diffusion}}{}}
\label{app:dps-algo}
\begin{algorithm}[H]
    \small
    \caption{DPS}
    \label{alg:dps_gauss}
    \begin{algorithmic}[1]
        \Require $N$, $\y$, $\{\zeta_i\}_{i=1}^N$, ${\{\tilde\sigma_i\}_{i=1}^N}$
        \State $\x_N \sim \Nc(\bm{0}, \bm{I})$
        \For{$i=N-1$ {\bfseries to} $0$}
            \State{$\hat\s \gets \s_\theta(\x_i, i)$}
            \State{$\hat\x_0 \gets \frac{1}{\sqrt{\bar\alpha_i}}(\x_i + (1 - \bar\alpha_i)\hat\s)$}
            \State{$\z \sim \Nc(\bm{0}, \bm{I})$}
            \State{$\x'_{i-1} \gets \frac{\sqrt{\alpha_i}(1-\bar\alpha_{i-1})}{1 - \bar\alpha_i}\x_i + \frac{\sqrt{\bar\alpha_{i-1}}\beta_i}{1 - \bar\alpha_i}\hat\x_0 + {\tilde\sigma_i \z}$}
            \State{$\x_{i-1} \gets \x'_{i-1} - {\zeta_i}\nabla_{\x_i}\|\y-\Ac(\hat\x_0)\|_2^2$}
        \EndFor
        \State {\bfseries return} $\hat{\rvx}_0$
    \end{algorithmic}
\end{algorithm}

Here $\hat\x_0$ is the Tweedie mean $\hat{x}_{t_i}(\x_i)$ of the iterate at step $i$, $\Ac:\R^d\to \R^L$ is a general linear or non-linear observation operator and $y\in \R^L$ is the observation. The guidance schedule $\{\zeta_i\}_{i=1}^N$ is taken to be trajectory-dependent,
\begin{equation}\label{eq:appendix_bias_schedule}
  \zeta_i=\frac{\alpha}{\|y-\Ac(\hat{x}_{t_i}(x_i))\|_2},
\end{equation}
where $\alpha\in [0.2,1]$ is a hyperparameter chosen depending on the inverse problem to be solved. In practice, the choice of guidance schedule significantly affects the performance of the algorithm.

\section{Proof of Theorem~\ref{thm:DPSbias}}\label{app:proof_DPSbias}
In this section, we will prove the following theorem
\mainthm*
\begin{proof}[Sketch]
The proof has three steps (elaborated below). 

\textbf{Step 1.} We record an evolution equation satisfied by any tilted prior path $t\mapsto \ora\mu_t\coloneqq h_t\rho_t/Z_t$. 
\begin{equation*}
\partial_t \ora\mu_t =\Delta\ora\mu_t + \nabla\cdot(x\ora\mu_t) + c[h_t,\rho_t](x)\,\ora\mu_t,
\end{equation*}
for an appropriate choice of $c[h_t,\rho_t]$ (\cref{lem:tilted_FP}).

\textbf{Step 2.} Specializing to the DPS tilt $h_t=e^{R_y\circ\hat{x}_t}$ and using the Kolmogorov backward equation for the conditional mean $\hat{x}_t(x)=\E[X_0\mid X_t=x]$, we identify the reaction term as exactly $c_{\DPS}$ from \eqref{define_c_DPS} (\cref{lem:c_dps_form}). 

\textbf{Step 3.} In \cref{lem:forward_ratio_FK}, we use \cref{lem:tilted_FP,lem:c_dps_form} and apply the Feynman--Kac formula \cref{lem:FK_ratio} to the resulting evolution equation to obtain an expression for the weights $\frac{\mu_y}{\nu^{\DPS}_y}$. Reversing the path integral, in \cref{lem:backward_ratio_fk} we recast the path expectation along the DPS reverse SDE \eqref{eq:DPS_SDE} with the forward OU path measure.
\end{proof}

\subsection{Two Tweedie identities and the Kolmogorov backward equation}\label{sec:identities}

We collect three facts used repeatedly. Throughout, $\rho_t$ denotes the marginal of the OU forward process \eqref{eq:SDE} starting at $X_0\sim\rho_*$, and $\hat{x}_t,\Sigma_t$ are the conditional mean and covariance of $X_0$ given $X_t$.

\textbf{(F1) First-order Tweedie.} A direct integration of the OU semigroup yields
\begin{equation}\label{eq:Tweedie_first_app}
\hat{x}_t(x) = e^{t}x + (e^{t}-e^{-t})\nabla\log\rho_t(x).
\end{equation}

\textbf{(F2) Second-order Tweedie.} Differentiating \eqref{eq:Tweedie_first_app} in $x$ and using the standard identity $\Sigma_t(x)=(e^{t}-e^{-t})^2 D^2\log\rho_t(x)+(e^{t}-e^{-t})e^{t} I$ (which follows from a second-order expansion of the OU posterior, or equivalently from differentiating Tweedie under Bayes' rule):
\begin{equation}\label{eq:Tweedie_second_app}
\nabla \hat{x}_t(x) = \frac{1}{e^{t}-e^{-t}}\,\Sigma_t(x).
\end{equation}
In particular, $\nabla\hat{x}_t$ is symmetric and positive semidefinite.

\textbf{(F3) Kolmogorov backward equation for $\hat{x}_t$.} By Anderson's reversal~\citep{anderson}, the time-reversed OU process $\ola{X}_s \coloneqq X_{T-s}$ satisfies the SDE
\begin{equation}\label{eq:OU_backward_app}
d\ola{X}_s = (\ola{X}_s + 2\nabla\log\rho_{T-s}(\ola{X}_s))\,ds + \sqrt{2}\,d\ola{B}_s,
\end{equation}
with generator acting on $f \in C^{\infty}_C$,
\begin{equation}\label{eq:OU_backward_generator}
\ola{L}^{OU}_t f \coloneqq \Delta f + (x+2\nabla\log\rho_{T-t}(x))\cdot\nabla f
\end{equation} 
Since
\[
\hat{x}_t(x) = \E[X_0\mid X_t = x] = \E[\ola{X}_T \mid \ola{X}_{T-t}=x],
\]
$\hat{x}_t$ is a Kolmogorov backward solution along $\ola{X}$, and hence
\begin{equation}\label{eq:KBE_xhat}
\partial_t \hat{x}_t(x) = \Delta \hat{x}_t(x) + \nabla\hat{x}_t(x)\cdot\big(x+2\nabla\log\rho_t(x)\big).
\end{equation}

\subsection{Step 1: Evolution of tilted prior paths}

\begin{lemma}[Tilted-prior Fokker--Planck]\label{lem:tilted_FP}
Let $h\in C^{1,2}([0,T]\times\R^d)$ be positive with $h_t\in L^1(\rho_t\,dx)$, and define the tilted prior path
\[
\ora\mu_t(x)\coloneqq\frac{h_t(x)\rho_t(x)}{Z_t},\qquad Z_t\coloneqq\int h_t(x)\rho_t(x)\,dx.
\]
Then $\ora\mu_t$ obeys
\begin{equation}\label{eq:source_FP}
\partial_t \ora\mu_t = L^{\dagger}\ora\mu_t + c[h_t,\rho_t](x)\,\ora\mu_t,
\end{equation}
where $L^{\dagger}f\coloneqq\Delta f + \nabla\cdot(xf)$ is the OU Fokker--Planck operator and the reaction term is
\begin{equation}\label{eq:c_general}
c[h_t,\rho_t](x)\coloneqq\partial_t\log h_t - \big(2\nabla\log\rho_t+x\big)\cdot\nabla\log h_t - \|\nabla\log h_t\|^2 - \Delta\log h_t - \tfrac{d}{dt}\log Z_t.
\end{equation}
\end{lemma}

\begin{proof}
Using the product rule and writing $Z_t$ terms as log-derivatives:
\begin{equation*}
    \partial_t\ora\mu_t
    = \frac{(\partial_t h_t)\rho_t}{Z_t} + \frac{h_t(\partial_t\rho_t)}{Z_t} - \Big(\tfrac{d}{dt}\log Z_t\Big)\ora\mu_t.
\end{equation*}
Writing $\partial_t h_t = h_t\,\partial_t\log h_t$ and substituting the OU Fokker--Planck equation $\partial_t\rho_t = L^\dagger\rho_t$:
\begin{equation*}
    \partial_t\ora\mu_t
    = (\partial_t\log h_t)\,\ora\mu_t
    + \frac{h_t}{Z_t}L^\dagger\rho_t
    - (\tfrac{d}{dt}\log Z_t)\,\ora\mu_t.
\end{equation*}

The OU Fokker--Planck operator is $L^\dagger\rho = \Delta\rho + \nabla\cdot(x\rho)$.
Using $\Delta\rho_t = \rho_t(\|\nabla\log\rho_t\|^2 + \Delta\log\rho_t)$
and $\nabla\cdot(x\rho_t) = \rho_t(d + x\cdot\nabla\log\rho_t)$:
\begin{equation*}
    h_t L^\dagger\rho_t
    = h_t\rho_t\Big[
        \|\nabla\log\rho_t\|^2
        + \Delta\log\rho_t
        + d
        + x\cdot\nabla\log\rho_t
    \Big].
\end{equation*}

Set $\phi = h_t/Z_t$ so that $\ora\mu_t = \phi\rho_t$.
We compute $L^\dagger\ora\mu_t = L^\dagger(\phi\rho_t)$ via the product rule applied to each term of $L^\dagger = \Delta + \nabla\cdot(x\,\cdot\,)$:
\begin{align*}
    \Delta(\phi\rho_t)
    &= \phi\Delta\rho_t + 2\nabla\phi\cdot\nabla\rho_t + \rho_t\Delta\phi, \\
    \nabla\cdot(x\phi\rho_t)
    &= \phi\,\nabla\cdot(x\rho_t) + \rho_t\,x\cdot\nabla\phi.
\end{align*}
Summing, we have
\begin{equation*}
    L^\dagger(\phi\rho_t)
    = \phi\, L^\dagger\rho_t
    + 2\nabla\phi\cdot\nabla\rho_t
    + \rho_t\Delta\phi
    + \rho_t\,x\cdot\nabla\phi.
\end{equation*}
Rearranging to isolate $\phi L^\dagger\rho_t$:
\begin{equation*}
    \phi\, L^\dagger\rho_t
    = L^\dagger(\phi\rho_t)
    - 2\nabla\phi\cdot\nabla\rho_t
    - \rho_t\!\left(\Delta\phi + x\cdot\nabla\phi\right).
\end{equation*}
Since $Z_t$ does not depend on $x$, we have $\nabla\phi = \nabla h_t/Z_t$ and $\Delta\phi = \Delta h_t/Z_t$.
Substituting $\phi\rho_t = \ora\mu_t$, $\nabla\rho_t = \rho_t\nabla\log\rho_t$, and $\rho_t/Z_t = \ora\mu_t/h_t$,
and using the identities $\nabla h_t/h_t = \nabla\log h_t$
and $\Delta h_t/h_t = \|\nabla\log h_t\|^2 + \Delta\log h_t$:
\begin{equation*}
    \frac{h_t}{Z_t}L^\dagger\rho_t
    = L^\dagger\ora\mu_t
    - \ora\mu_t\Big[
        2\nabla\log\rho_t\cdot\nabla\log h_t
        + \|\nabla\log h_t\|^2
        + \Delta\log h_t
        + x\cdot\nabla\log h_t
    \Big].
\end{equation*}

Collecting all $\ora\mu_t$ terms:
\begin{equation*}
    \partial_t\ora\mu_t = L^\dagger\ora\mu_t + c[h_t,\rho_t](x)\,\ora\mu_t,
\end{equation*}
where the reaction coefficient is
\begin{equation*}
    c[h_t,\rho_t](x)
    = \partial_t\log h_t
    - (2\nabla\log\rho_t+x)\cdot\nabla\log h_t
    - \|\nabla\log h_t\|^2
    - \Delta\log h_t
    - \tfrac{d}{dt}\log Z_t
\end{equation*}
\end{proof}

\subsection{Step 2: Specialization to the DPS tilt}

\begin{lemma}[Reaction term for the DPS tilt]\label{lem:c_dps_form}
With $h_t(x) = \exp(R_y(\hat{x}_t(x)))$, so that $\overrightarrow{\mu}_t = \overrightarrow{\mu}^{\DPS}_t$ in \eqref{eq:DPScurve}, the reaction $c[h_t,\rho_t]$ from \eqref{eq:c_general} reduces to
\begin{equation}\label{eq:c_h_simplified}
c[h_t,\rho_t](x) = -\frac{1}{(e^{t}-e^{-t})^2}\Big[\trace\big(\Sigma_t(x)\,D^2R_y(\hat{x}_t(x))\,\Sigma_t(x)\big) + \|\Sigma_t(x)\nabla R_y(\hat{x}_t(x))\|^2\Big]-\tfrac{d}{dt}\log Z_t,
\end{equation}
that is, $c[h_t,\rho_t](x)=c_{\DPS}(t,x)$ as defined in \eqref{define_c_DPS}.
\end{lemma}

\begin{proof}
Set $\log h_t = R_y\circ\hat{x}_t$ and apply the chain rule component-wise:
\begin{align*}
\partial_t\log h_t &= \nabla R_y(\hat{x}_t)\cdot\partial_t\hat{x}_t,\\
\nabla\log h_t &= (\nabla\hat{x}_t)\,\nabla R_y(\hat{x}_t),\\
\Delta\log h_t &= \nabla R_y(\hat{x}_t)\cdot\Delta\hat{x}_t + \trace\!\big(\nabla\hat{x}_t\,D^2R_y(\hat{x}_t)\,\nabla\hat{x}_t\big),
\end{align*}
using symmetry of $\nabla\hat{x}_t$ from~\eqref{eq:Tweedie_second_app}. Substituting into~\eqref{eq:c_general} and grouping terms by their dependence on $\nabla R_y$ and $D^2R_y$, we get the following expression for $c[h_t,\rho_t]$:
\begin{align*}
    c[h_t,\rho_t](x)+\tfrac{d}{dt}\log Z_t
    &= \nabla R_y(\hat{x}_t)\cdot\partial_t\hat{x}_t
    - (2\nabla\log\rho_t+x)\cdot(\nabla\hat{x}_t)\,\nabla R_y(\hat{x}_t)\\
    & \hspace{2cm}
    - \|(\nabla\hat{x}_t)\,\nabla R_y(\hat{x}_t)\|^2
    - \nabla R_y(\hat{x}_t)\cdot\Delta\hat{x}_t - \trace\!\big(\nabla\hat{x}_t\,D^2R_y(\hat{x}_t)\,\nabla\hat{x}_t\big)
\end{align*}
Grouping the terms of $\nabla R$ together, we have
\begin{align*}
    c[h_t,\rho_t](x)+\tfrac{d}{dt}\log Z_t
    &= \nabla R_y(\hat{x}_t)\cdot[\partial_t\hat{x}_t-\Delta \hat{x}_t- (2\nabla\log\rho_t+x)\cdot(\nabla\hat{x}_t)] \\
    &\hspace{2cm} - \|(\nabla\hat{x}_t)\,\nabla R_y(\hat{x}_t)\|^2 - \trace\!\big(\nabla\hat{x}_t\,D^2R_y(\hat{x}_t)\,\nabla\hat{x}_t\big)
\end{align*}
By the Kolmogorov backward equation~\eqref{eq:KBE_xhat}, we have $\partial_t\hat{x}_t - \nabla\hat{x}_t\,(x+2\nabla\log\rho_t) - \Delta\hat{x}_t = 0$. This is the key cancellation underlying the bias formula. The remaining contributions are $-\|\nabla\log h_t\|^2$ and the trace piece of $-\Delta\log h_t$:
\[
-\|(\nabla\hat{x}_t)\nabla R_y(\hat{x}_t)\|^2 - \trace\!\big(\nabla\hat{x}_t\,D^2R_y(\hat{x}_t)\,\nabla\hat{x}_t\big).
\]
Using~\eqref{eq:Tweedie_second_app} to substitute $\nabla\hat{x}_t = \Sigma_t/(e^{t}-e^{-t})$ and pulling out the common scalar factor yields~\eqref{eq:c_h_simplified}.
\end{proof}

\subsection{Step 3: Feynman--Kac two ways}
We now combine \cref{lem:tilted_FP,lem:c_dps_form} to prove \cref{thm:DPSbias}: part (i) in \cref{lem:forward_ratio_FK} and part (ii) in \cref{lem:backward_ratio_fk}. 

\begin{lemma}\label{lem:forward_ratio_FK}
The terminal law $\nu^{\DPS}_y \defeq \overleftarrow{\nu}^{\DPS}_T$ of the DPS-SDE \eqref{eq:DPS_SDE} differs from the true posterior $\mu_y$ by a pointwise multiplicative weight:
\begin{equation*}
  \mu_y(x) = \omega(x)\,\nu^{\DPS}_y(x).
\end{equation*}
The weight $\omega$ admits a Feynman--Kac representation in terms of the reaction term $c_{\DPS}$ defined in \eqref{define_c_DPS}\textup{:}
\begin{equation*}
    \omega(x) = \E_{Y\sim \eqref{eq:DPS_SDE}}\!\left[\frac{\overrightarrow{\mu}^{\DPS}_T(Y_0)}{\gamma(Y_0)}\,\exp\!\Bigl({-\int_{0}^{T}c_{\DPS}(T-s,Y_s)\,ds}\Bigr)\,\Bigm|\,Y_T=x\right].
  \end{equation*}
\end{lemma}
\begin{proof}
Writing $\overrightarrow{\mu}_t$ for $\overrightarrow{\mu}^{\DPS}_t$ in this proof, setting $\overleftarrow{\mu}_t \coloneqq \overrightarrow{\mu}_{T-t}$, and applying \eqref{eq:source_FP} together with the diffusion-to-drift identity
\[
-L^{\dagger}\overleftarrow{\mu}_t = \Delta\overleftarrow{\mu}_t-\nabla\!\cdot\!\big((x+2\nabla\log\overleftarrow{\mu}_t)\,\overleftarrow{\mu}_t\big)
\]
gives exactly 
\begin{equation}\label{eq:truereversalapp}
\begin{cases}\partial_t\overleftarrow{\mu}_t=\Delta\overleftarrow{\mu}_t-2\,\nabla\!\cdot\!\big(\nabla\log\overrightarrow{\mu}_{T-t}\,\overleftarrow{\mu}_t\big)-\nabla\!\cdot(x\,\overleftarrow{\mu}_t)-c_{\DPS}(T-t,x)\,\overleftarrow{\mu}_t,\\
  \overleftarrow{\mu}_0(x)=\overrightarrow{\mu}_T(x),
\end{cases}
\end{equation}
with reaction $-c_{\DPS}$ by \cref{lem:c_dps_form}, which is \eqref{eq:truereversal}. By construction, $\overleftarrow{\mu}_T = \overrightarrow{\mu}_0 = \mu_y$. Equation \eqref{eq:truereversalapp} is the Fokker--Planck equation associated with the DPS reverse SDE 
\begin{equation}\label{eq:DPS_SDEapp}
\begin{cases}
dY_t = \big(Y_t+2\nabla\log\rho_{T-t}(Y_t)+\tfrac{2}{e^{T-t}-e^{-(T-t)}}\,\Sigma_{T-t}(Y_t)\,\nabla R_y(\hat{x}_{T-t}(Y_t))\big)\,dt+\sqrt{2}\,dB_t,\\[2pt]
Y_0\sim \gamma,
\end{cases}
\end{equation}
augmented by a multiplicative reaction $-c_{\DPS}$. The DPS algorithm itself directly simulates \eqref{eq:DPS_SDEapp} from $Y_0\sim\gamma$, that is, \emph{without} the reaction term. This produces a marginal $\overleftarrow{\nu}_t$ with $\overleftarrow{\nu}_T = \nu^{\DPS}_y$. Effectively, the DPS algorithm solves
\begin{equation}\label{eq:DPS_SDEapp2}
\begin{cases}
\partial_t\overleftarrow{\nu}_t = \Delta\overleftarrow{\nu}_t -2\nabla\!\cdot(\nabla\log\overrightarrow{\mu}_{T-t}\,\overleftarrow{\nu}_t)
 -\nabla\!\cdot\!\big(x\,\overleftarrow{\nu}_t\big),\\[2pt]
\overleftarrow{\nu}_0(x)=\gamma(x),
\end{cases}
\end{equation}
Two operators \eqref{eq:truereversalapp} and \eqref{eq:DPS_SDEapp2} differ only by a multiplicative reaction and an initial condition. They can be related by a Feynman--Kac formula of \cref{lem:FK_ratio}. The ground truth satisfies $\mu_y(x)=\overleftarrow{\mu}_T(x)$, for any test function $\varphi(x)$ we have
\begin{align*}
\int_{\R^d}\varphi(x)\mu_y(x)\,dx
&=\int_{\R^d}\varphi(x)\overleftarrow{\mu}_T(x)\,dx\\
&= \E_{Y\sim\eqref{eq:DPS_SDEapp}}\!\left[\phi(Y_T)
\frac{\overrightarrow{\mu}_T(Y_0)}{\gamma(Y_0)}\,e^{-\int_{0}^{T}c_{\DPS}(T-s,Y_s)\,ds}
\right]\\
&= \E_{Y\sim\eqref{eq:DPS_SDEapp}}\left[\phi(Y_T)\E\underbrace{\left[ \frac{\overrightarrow{\mu}_T(Y_0)}{\gamma(Y_0)}\,e^{-\int_{0}^{T}c_{\DPS}(T-s,Y_s)\,ds}| Y_T\right]}_{\omega(Y_T)} \right]\\
&= \int_{\R^d}\varphi(x)\,\omega(x)\,\nu^{\DPS}_y(x)\,dx,
\end{align*}
where we conditioned on the value of $Y_T$, used the law of total expectation, and the observation $Y_T\sim \overleftarrow{\nu}_T$. Concretely,
\begin{equation}\label{eq:FK_DPSrev}
\frac{\mu_y(x)}{\nu^{\DPS}_y(x)}=\frac{\overleftarrow{\mu}_T(x)}{\overleftarrow{\nu}_T(x)}
=
\E_{Y\sim\eqref{eq:DPS_SDEapp}}\!\left[
\frac{\overrightarrow{\mu}_T(Y_0)}{\gamma(Y_0)}\,e^{-\int_{0}^{T}c_{\DPS}(T-s,Y_s)\,ds}\,\Bigm|\,Y_T=x
\right].
\end{equation}
The boundary factor $\overrightarrow{\mu}_T/\gamma$ accounts for the mismatch between $\overleftarrow{\mu}_0=\overrightarrow{\mu}_T$ and $\overleftarrow{\nu}_0=\gamma$.
\end{proof}
The formula for the weight using OU, can now be derived using Anderson the time reversal of SDEs. For clarity and brevity, we instead provide a derivation using PDE satisfied by the ratio of the algorithmic path to the PDE for the ratio of the algorithmic path to the surrogate path. As in \eqref{surrogate} and \eqref{algorithm} we here define 
\begin{equation*}
\begin{split}
& \overleftarrow{\mu}_t \defeq \overleftarrow{\mu}_t^{\DPS} = \frac{1}{Z_{T-t}}\exp(R_y \circ \hat{x}_{T-t}(x))\rho_{T-t}(x) \\ 
& \overleftarrow{\nu}_t \defeq \overleftarrow{\nu}^{\DPS}_t = \mathrm{Law}(Y_t).
\end{split}
\end{equation*}
and for $0 \leq t < T$, as in \cref{app:FK_background} define the density ratio $\psi_t(x)$, 
\begin{equation}
\psi_t \overleftarrow{\mu}_t = \overleftarrow{\nu}_t
\end{equation}
noting $\lim_{t \to T} \psi_t = \psi_T = \frac{1}{\omega(x)}$ as in 
\eqref{eq:weight_DPS_forward}.
\begin{lemma}\label{lem:backward_ratio_fk}
The ratio $\psi_t(x)$ solves the parabolic initial value problem,
\begin{equation}\label{eqn:psi_equation}
  \begin{cases}
    \partial_t \psi_t = \laplacian \psi_t - x\cdot\grad \psi_t + c_{\DPS}(T-t,x)\,\psi_t \\
    \psi_0(x) = \frac{\gamma(x)}{\overrightarrow{\mu}^{\DPS}_T(x)}, \quad
  \end{cases}
\end{equation}
and thus, by Feynman--Kac formula, with $t < T$,
\begin{equation}
    \psi_t(x) = \E_{X\sim\mathrm{OU}}\left[\frac{\gamma(X_t)}{\overrightarrow{\mu}^{\DPS}_T(X_t)}\exp\left(\int_0^t c_{\DPS}(T-t+s,X_s)\,ds\right) \bigg\vert X_0 = x\right]
\end{equation} 
Taking the limit $t \to T$ yields the statement of \cref{thm:DPSbias} with $\frac{1}{\omega(x)} = \psi_T(x)$ as the ratio. 
\end{lemma}

\begin{proof}
The following calculations hold classically, since for $t < T$, both $\overleftarrow{\mu}_t(x)$ and $\overleftarrow{\nu}^{\DPS}_t(x)$ are smooth, positive densities. Note $\log \psi_t = \log \overleftarrow{\nu}_t- \log \overleftarrow{\mu}_t$, using \eqref{eq:truereversalapp} and \eqref{eq:DPS_SDEapp2} we have
\begin{equation*}
\begin{split}
\partial_t \log \overleftarrow{\nu}_t &= \laplacian \log \overleftarrow{\nu}_t  + \|\grad \log \overleftarrow{\nu}_t\|^2 -d- 2 \Delta \log \overleftarrow{\mu}_t \\
&\qquad- 
x\cdot\grad \log \overleftarrow{\nu}_t-2 \grad \log \overleftarrow{\mu}_t\cdot\grad \log \overleftarrow{\nu}_t\\
 \partial_{t} \log \overleftarrow{\mu}_t &= \laplacian \log\overleftarrow{\mu}_t +  \|\grad \log \overleftarrow{\mu}_t\|^2  -d- 2 \Delta \log \overleftarrow{\mu}_t\\
 &\qquad - 
x\cdot\grad \log \overleftarrow{\mu}_t-2\|\grad \log \overleftarrow{\mu}_t\|^2  - c_{\DPS}(T-t,\cdot) 
\end{split} 
\end{equation*} 
Taking the difference and completing the square shows
\begin{equation}\label{eqn:psi_ivp}
\begin{split}
\partial_t \log \psi_t &= \partial_t\log\overleftarrow{\nu}_t-  \partial_{t} \log \overleftarrow{\mu}_t\\ 
& = \laplacian \log \psi_t+\|\grad \log \psi_t\|^2 -x\cdot\grad \log \psi_t + c_{\DPS}(T-t,\cdot)
\end{split}
\end{equation}
Applying the Cole--Hopf transformation ($\log \psi_t \overset{\exp(\cdot)}{\mapsto} \psi_{t}$) obtains \eqref{eqn:psi_equation}. Finally, applying Feynman--Kac to an initial-value problem (rather than terminal-value) induces time-reversal of the reaction term $c_{\DPS}(T-t+s,X_s)$ in the path-functional.
\end{proof}
\section{Early Guidance Stopping, Proof of \cref{thm:early}}\label{app:proof_earl_stopping}
\begin{algorithm}[H]
    \small
    \caption{DPS with Early Guidance Stopping}
    \label{alg:early_stopping}
    \begin{algorithmic}[1]
        \Require $i_{\mathrm{stop}}$, $N$, $\y$, $\{\zeta_i\}_{i=1}^N$, $\{\tilde\sigma_i\}_{i=1}^N$
        \State $\x_N \sim \Nc(\bm{0}, \bm{I})$
        \For{$i=N-1$ {\bfseries to} $0$}
            \State $\hat\s \gets \s_\theta(\x_i, i)$
            \State $\hat\x_0 \gets \frac{1}{\sqrt{\bar\alpha_i}}\big(\x_i + (1 - \bar\alpha_i)\hat\s\big)$
            \State $\z \sim \Nc(\bm{0}, \bm{I})$
            \State $\x'_{i-1} \gets \frac{\sqrt{\alpha_i}(1-\bar\alpha_{i-1})}{1 - \bar\alpha_i}\x_i + \frac{\sqrt{\bar\alpha_{i-1}}\beta_i}{1 - \bar\alpha_i}\hat\x_0 + \tilde\sigma_i \z$
            \If{$i > i_{\mathrm{stop}}$}
                \State $\x_{i-1} \gets \x'_{i-1} - \zeta_i\,\nabla_{\x_i}\|\y-\Ac(\hat\x_0)\|_2^2$
            \Else
                \State $\x_{i-1} \gets \x'_{i-1}$
            \EndIf
        \EndFor
        \State \Return $\hat\x_0$
    \end{algorithmic}
\end{algorithm}
The output \cref{alg:early_stopping}, up to discretization error, is characterized in the following result.
\earlythm*
\begin{proof}
We apply~\eqref{surrogate} to the annealed path~\eqref{eq:annealedpath}:
\[
t\mapsto \ora\mu_t \coloneqq \ora\mu^{\eta}_t = \frac{e^{\alpha\eta_t R_y(\hat{x}_t(x))}\,\rho_t}{Z_t}.
\]
By~\cref{lem:tilted_FP}, the reverse equation reads
\begin{align*}
     \partial_t\ola\mu_t &= \Delta\ola\mu_t -\nabla\cdot(x\,\ola\mu_t) - 2\nabla\cdot(\nabla\log\rho_{T-t}\,\ola\mu_t) - 2\alpha\eta_{T-t}\,\nabla\cdot(\nabla (R_y\circ \hat{x}_{T-t})\,\ola\mu_t) \\
        &\qquad- c^*_{\DPS}(T-t,x)\,\ola\mu_t,
\end{align*}
where $R_y(x) = -\|\mathcal{A}(x)-y\|_2$ for the (possibly non-linear) observation operator $\mathcal{A}$, and $c^*_{\DPS}$ is the reaction of \cref{thm:early}: \cref{lem:c_dps_form} applied to the tilt $h_t=e^{\alpha\eta_t R_y\circ\hat{x}_t}$, i.e.\ with $R_y$ replaced by $\alpha\eta_t R_y$, plus the term $\alpha\tfrac{d\eta_t}{dt}\,R_y\circ\hat{x}_t$ coming from $\partial_t\log h_t$.

The corresponding algorithmic SDE~\eqref{algorithm} with early stopping at time $t_{\mathrm{stop}} \coloneqq T - t_*$ is
\begin{equation}
    \partial_t\ola\nu_t=\begin{cases}
        \Delta\ola\nu_t -\nabla\cdot(x\,\ola\nu_t) - 2\nabla\cdot(\nabla\log\rho_{T-t}\,\ola\nu_t) - 2\alpha\eta_{T-t}\,\nabla\cdot(\nabla (R_y\circ \hat{x}_{T-t})\,\ola\nu_t), & t\in (0,t_{\mathrm{stop}}),\\
        \Delta\ola\nu_t -\nabla\cdot(x\,\ola\nu_t) - 2\nabla\cdot(\nabla\log\rho_{T-t}\,\ola\nu_t), & t\in [t_{\mathrm{stop}},T).
    \end{cases}
\end{equation}
On $(0, t_{\mathrm{stop}})$, both $\ola\mu_t$ and $\ola\nu_t$ satisfy the same equation, so~\cref{thm:DPSbias} applies directly and yields
\[
\frac{\ola\nu_{t_{\mathrm{stop}}}(x)}{\ola\mu_{t_{\mathrm{stop}}}(x)} \;=\; \frac{1}{\omega_{t_*}}(x) \;=\; \E_{X\sim\mathrm{OU}}\!\left[\frac{\gamma(X_{T-t_*})}{\ora\mu_T(X_{T-t_*})}\,\exp\!\left(\int_0^{T-t_*} c^*_{\DPS}(t_*+s,\,X_s)\,ds\right) \,\bigg|\, X_0 = x\right].
\]
Substituting the explicit form of $\ola\mu_{t_{\mathrm{stop}}}$,
\begin{equation}\label{eq:nu_init}
\ola\nu_{T-t_*}(x) \;=\; \frac{1}{\omega_{t_*}}(x)\,\ola\mu_{t_{\mathrm{stop}}}(x) \;=\; \frac{1}{\omega_{t_*}}(x)\, \frac{e^{\alpha\eta_{t_*}R_y(\hat{x}_{t_*}(x))}}{Z_{t_*}}\,\ola\rho_{T-t_*}(x).
\end{equation}

On $[t_{\mathrm{stop}}, T)$, the SDE for $\ola\nu_t$ is the unbiased reverse OU equation, started from~\eqref{eq:nu_init}. Setting $s = T - t$ for the corresponding forward time, the Radon--Nikodym derivative of the initial condition with respect to the OU forward marginal $\rho_{t_*}$ is

\[
g_{t_*}(x) \;\coloneqq\; \frac{d\ola\nu_{T-t_*}}{d\ola\rho_{T-t_*}}(x) \;=\; \frac{1}{\omega_{t_*}(x)}\,\frac{e^{\alpha\eta_{t_*}R_y(\hat{x}_{t_*}(x))}}{Z_{t_*}}.
\]
Applying the Feynman--Kac identity for ratios~(\cref{lem:FK_ratio}),
\begin{equation}\label{eq:final_density}
\frac{d\ola\nu_T}{d\rho_*}(x_0)=\frac{d\ola\nu_T}{d\ola\rho_T}(x_0) \;=\; \E_{X\sim\mathrm{OU}}\!\left[\,g_{t_*}(X_{t_*})\,\Big|\,X_0 = x_0\right],
\end{equation}
where $\{X_s\}_{s\ge 0}$ is the OU process started from $X_0 \sim \rho_*$. Substituting the expression for $g_{t_*}$ yields the claim.
\end{proof}

\section{Time discretization of DDPM}\label{sec:DDPM}
To establish the correspondence between the discrete variance schedule $\beta_i$ used in Denoising Diffusion Probabilistic Models (DDPM) \cite{DDPM} and the continuous time steps $\Delta t_i$ of the underlying Ornstein--Uhlenbeck (OU) process, we compare their respective transition kernels. 

The forward Markov jump process in DDPM defines the transition from step $i$ to $i+1$ as:
\begin{equation}
q(x_{i+1} | x_i) = \mathcal{N}(x_{i+1}; \sqrt{1 - \beta_i}x_i, \beta_i I)
\end{equation}
The continuous-time forward (OU) SDE under consideration is given by:
\begin{equation}
dX_t = -X_t dt + \sqrt{2} dB_t
\end{equation}
For a finite time increment $\Delta t_i$, the exact solution to this SDE yields the transition:
\begin{equation}
p(x_{t+\Delta t_i} | x_t) = \mathcal{N}(x_{t+\Delta t_i}; e^{-\Delta t_i}x_t, (1 - e^{-2\Delta t_i})I)
\end{equation}
For the discrete Markov chain to exactly discretize the continuous SDE, the coefficients of the mean and variance must be consistent across regimes:
\begin{align}
e^{-\Delta t_i} = \sqrt{1 - \beta_i} \qquad 1 - e^{-2\Delta t_i} = \beta_i.
\end{align}
Solving for $\Delta t_i$ we obtain
\begin{align}
\Delta t_i &= -\frac{1}{2} \log(1 - \beta_i)
\end{align}

The linear noise schedule of DDPM is given by 
\[ 
\beta_i=\beta_{min} + i\frac{\beta_{max}-\beta_{min}}{N}\approx 10^{-4} + 2i\, 10^{-5}\qquad\mbox{for }i=1,\ldots, 1000, 
\]  
with the choices $\beta_{min} = 10^{-4}$, $\beta_{max} = 0.02$, and $N=1000$ steps. Applying the first-order Taylor expansion $\log(1 - \epsilon) \approx -\epsilon$, we obtain the approximately linear relationship between $\Delta t_i$ and $\beta_i$: 
\begin{equation}\label{eq:linear_approximation_step_size} 
\Delta t_i \approx \frac{1}{2} \beta_i 
\end{equation} 
Next, we derive a relationship in time between the discrete steps $t_i$ and the continuous time $t$ by summing over the increments: 
\begin{align} 
t_i &= \sum_{j=1}^i \Delta t_j \approx \frac{1}{2} \sum_{j=1}^i \beta_j = \frac{1}{2} \sum_{j=1}^i \left(10^{-4} + 2j\, 10^{-5}\right) = \frac{1}{2} \left(10^{-4}i + 2 \cdot 10^{-5} \frac{i(i+1)}{2}\right). 
\end{align} 
Next, we solve for a function $i(t)$ that maps continuous time to discrete steps by inverting the quadratic relationship $ \frac{1}{2} \left(10^{-4}i + 2 \cdot 10^{-5} \frac{i(i+1)}{2}\right) = t$: 
\begin{align} 
i(t) &= \frac{\sqrt{121 + 8 t\, 10^5 }-11}{2} 
\end{align} 
This function $i(t)$ provides a mapping from continuous time $t$ to the corresponding discrete step index $i$ in the DDPM framework, allowing us to understand the time step behavior in the continuum limit. Substituting $i(t)$ into~\eqref{eq:linear_approximation_step_size}, we obtain
\begin{align}
\Delta t(t) 
&\approx \frac{1}{2}\beta_{i(t)} 
= \frac{1}{2}\left(10^{-4} + 2\cdot 10^{-5}\, i(t)\right) \notag\\
&= \frac{1}{2}\left(10^{-4} + 10^{-5}\left(\sqrt{121 + 8\cdot 10^5\, t} - 11\right)\right) \notag\\
&= \frac{10^{-5}}{2}\left(\sqrt{121 + 8\cdot 10^5\, t} - 1\right).
\end{align}
A naive large-$t$ approximation $\Delta t(t)\sim\sqrt{2\cdot 10^{-5}\,t}$ would incorrectly vanish at $t=0$. To preserve the nonzero constant floor at the origin, we drop the small $-1$ term (negligible compared to $\sqrt{121}=11$) but keep the constant $121$ inside the square root. Pulling the prefactor $\tfrac{10^{-5}}{2}$ inside the radical yields the compact form
\begin{equation}\label{eq:delta_t_approx}
\Delta t(t) \approx \sqrt{2\cdot 10^{-5}\, t + \Delta t_0^2},\qquad \Delta t_0 = 5.5\cdot 10^{-5}.
\end{equation}
\section{Forward Euler instability}\label{sec:instability}
In terms of implementation, DPS~\cref{alg:dps_gauss} integrates the different terms of~\eqref{eq:DPS_SDE} in different ways. The denoising step, corresponding to the terms $\big(Y_t+2\nabla\log\rho_{T-t}(Y_t)\big)\,dt+\sqrt{2}\,dB_t$, is integrated implicitly via DDPM in Step~6, avoiding numerical instabilities. The guidance term $\alpha\eta_{T-t}\,\tfrac{2}{e^{T-t}-e^{-(T-t)}}\,\Sigma_{T-t}(Y_t)\,\nabla R_y(\hat{x}_{T-t}(Y_t))$, however, is integrated explicitly via forward Euler in Step~7. The \emph{annealing schedule} $\eta_t = 1/\Delta t(t)$ in~\eqref{eq:annealingschedule} is an auxiliary quantity we introduce to compensate for the missing time step $\Delta t$ in Step~7. In place of $\Delta t$, the algorithm multiplies by the path-dependent factor
\[
\zeta_i=\frac{\alpha}{\|y-\mathcal{A}(\hat{x}_{t_i}(x_i))\|_2},
\]
applied to the gradient of the squared residual $\|y-\mathcal{A}(\hat{x}_{t_i}(x_i))\|_2^2$ in Step~7. Since $\nabla\|\cdot\|_2^2=2\|\cdot\|_2\nabla\|\cdot\|_2$, one guidance step equals $2\alpha\,\nabla(R_y\circ\hat{x}_{t_i})(x_i)$ with the reward $R_y(x)=-\|\mathcal{A}(x)-y\|_2$ of~\eqref{eq:annealedpath}. In our analysis, this is the guidance drift $2\alpha\eta_t\nabla(R_y\circ\hat{x}_t)$ of the annealed path~\eqref{eq:annealedpath} integrated over one step, since $\eta_t\,\Delta t=1$ for the annealing schedule~\eqref{eq:annealingschedule}.

To derive the oscillations, we assume that the prior $\rho_*$ is a smooth distribution supported on a smooth lower-dimensional manifold $\mathcal{M}\subset\R^d$ embedded in the ambient space. In the limit $t\to 0^+$,
\[
\hat{x}_t(x)\to P_\mathcal{M}(x),\qquad \nabla\hat{x}_t(x)\to P_{T_{P_\mathcal{M}(x)}\mathcal{M}},
\]
where $P_\mathcal{M}$ denotes the orthogonal projection onto $\mathcal{M}$ and $P_{T_{P_\mathcal{M}(x)}\mathcal{M}}$ denotes the orthogonal projection onto the tangent space at $P_\mathcal{M}(x)$, see \cref{sec:zeronoiselimit}. Consequently, as $t\to 0^+$, the guidance is well-approximated by a flow on $\mathcal{M}$,
\[
dY_t=-\frac{1}{\Delta t_0}\, P_{T_{Y_t}\mathcal{M}}\, \nabla\|\mathcal{A}(Y_t)-y\|_2\, dt,
\]
where $P_{T_{Y_t}\mathcal{M}}$ denotes the orthogonal projection onto the tangent space at $Y_t$. Taking a forward Euler step of size $\Delta t_0$, the prefactor $1/\Delta t_0$ cancels the step size exactly, so one Euler step corresponds to one full projected-gradient step on $\mathcal{M}$, independently of $\Delta t_0$.

\textbf{Local Lipschitz constant.}\; The local Lipschitz constant of the projected drift on $\mathcal{M}$ scales as
\begin{equation}\label{eq:lip}
L_{\mathrm{lip}} \;\sim\; \frac{1}{\Delta t_0}\,\frac{\sigma_{\max}\!\left(\nabla\mathcal{A}\,P_{T_Y\mathcal{M}}\right)^{2}}{\|\mathcal{A}(Y)-y\|_2},
\end{equation}
where the residual in the denominator originates from the gradient of the unsquared norm, which renormalizes to unit magnitude as $Y$ approaches the constraint.

\textbf{Stability criterion and inevitable oscillations.}\; Forward Euler stability requires $\Delta t_0 \cdot L_{\mathrm{lip}}\le 2$, i.e.,
\begin{equation}\label{eq:stability}
\sigma_{\max}\!\left(\nabla\mathcal{A}\,P_{T_Y\mathcal{M}}\right)^{2}\;\le\;2\,\|\mathcal{A}(Y) - y\|_2.
\end{equation}
As $Y$ approaches the constraint set $\{Y : \mathcal{A}(Y)=y\}$, the right-hand side tends to zero while the left-hand side depends only on $\mathcal{A}$ and the local geometry of $\mathcal{M}$. The criterion is therefore inevitably violated near the constraint. This is the standard pathology of forward Euler applied to the unsquared norm: the gradient does not vanish as the residual shrinks, but merely renormalizes to unit magnitude along $\nabla\mathcal{A}^\top(\mathcal{A}(Y)-y)/\|\mathcal{A}(Y)-y\|_2$. The iteration overshoots and oscillates around the constraint, and no choice of step size can restore stability/convergence. We note that these oscillations occur only parallel to the data manifold. Implicit integration avoids these instabilities, see \cref{fig:instability_iterations}.

\raggedbottom
\raggedbottom
\needspace{12\baselineskip}
\section{Empirical Evidence of Instabilities}\label{app:instability}
\textbf{LQG instability.}\;
\emph{Setup.} Take the prior $\rho_*=\mathcal{N}(\mu,\sigma^2 I)$ on $\R^2$ with $\mu=(1,1)$ and
$\sigma^2=0.05$, and observe the second coordinate, $\mathcal{A}(x)=a^\top x$ with $a=(0,1)$ and
$y=1.5$. The constraint line $\{x_2=y\}$ sits $2.2\sigma$ into the tail of the prior. We run
$N=1000$ steps of the linear schedule of \cref{sec:DDPM} with guidance strength $\alpha=0.35$.
This experiment steps along the gradient of the unsquared residual, half the step of \cref{alg:dps_gauss}, so under the schedule~\eqref{eq:bias_schedule} DPS targets~\eqref{eq:target} with $\alpha/2$ in place of $\alpha$, which here is
\[
\mu^{\mathrm{Target}}_y(x)\;\propto\; e^{-\frac{\alpha}{2}\eta_0\,|x_2-y|}\,\rho_*(x),
\qquad \tfrac{\alpha}{2}\eta_0=\alpha/(2\Delta t_0)\approx 3.2\times10^{3}:
\]
a Laplace tilt of the prior which, to within $10^{-3}$, is the prior restricted to the line, i.e.\ $x_2=y$ and $x_1\sim\mathcal{N}(1,\sigma^2)$.

\emph{The DPS SDE.} Write $f_s(x)\defeq\alpha\,|y-a^\top\hat{x}_s(x)|$. With the
schedule~\eqref{eq:bias_schedule}, the guidance term of~\eqref{eq:DPS_SDE} becomes a gradient
flow on $f$ scaled by the annealing schedule $\eta_t=1/\Delta t(t)$ (\cref{sec:instability}):
\[
dY_t=\big(Y_t+2\nabla\log\rho_{T-t}(Y_t)\big)\,dt\;-\;\eta_{T-t}\,\nabla f_{T-t}(Y_t)\,dt\;+\;\sqrt{2}\,dB_t .
\]
DPS integrates the first term with the DDPM step, $\tilde{Y}_{i-1}=\mathrm{DDPM}(Y_i)$, and the
guidance term with forward Euler over the step's OU time $\Delta t_i$:
\[
Y_{i-1}=\tilde{Y}_{i-1}-\Delta t_i\,\eta_i\nabla f_i(Y_i)=\tilde{Y}_{i-1}-\nabla f_i(Y_i)
=\tilde{Y}_{i-1}+\alpha\,\operatorname{sign}(r_i)\,J_i^\top a,
\]
where $r_i=y-a^\top\hat{x}_{t_i}(Y_i)$ and $J_i=\nabla\hat{x}_{t_i}(Y_i)$. The step size cancels the
schedule, so every step has fixed length $\alpha\|J_i^\top a\|\to\alpha$ however small the residual.
Once $|r_i|<\alpha$ the step overshoots, the sign flips, and the next step returns: a
period-$2$ orbit of amplitude $\alpha$ about the constraint. Here $\alpha=0.35$ exceeds the prior
width $\sigma\approx0.22$. \Cref{fig:instability_iterations}(a) shows this orbit on three particles.

\begin{figure}[!tb]
    \centering
    \includegraphics[width=\linewidth]{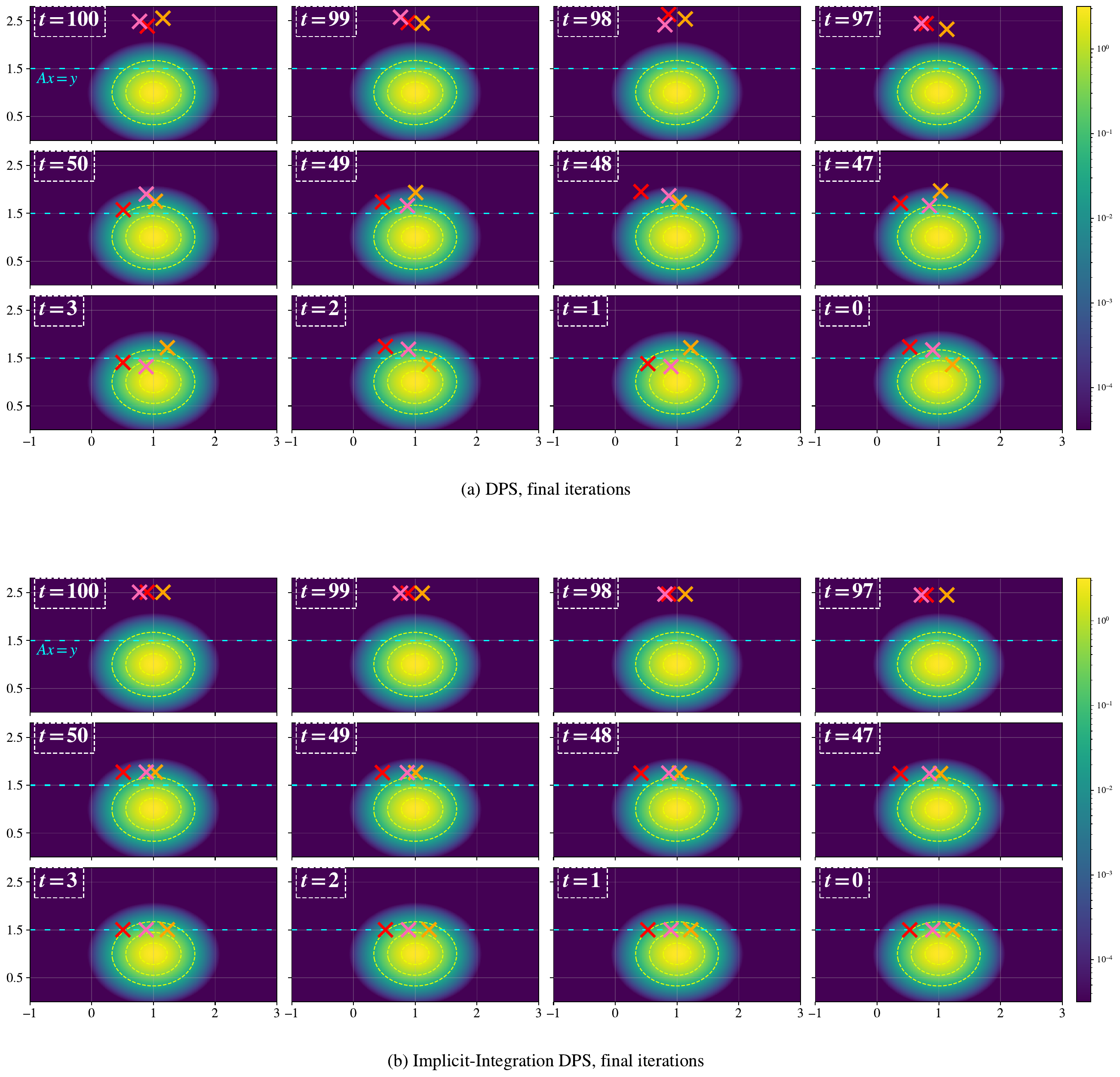}
    \caption{Three particles under \textbf{(a)} DPS and \textbf{(b)} implicitly integrated DPS
    (same prior, measurement, $\alpha=0.35$, noise schedule and seed). Rows show steps
    $i=100$--$97$ and $3$--$0$ ($i=0$ is the output), over the prior density (log
    scale) with its $1,2,3\sigma$ contours dashed; the cyan line is $\{x_2=y\}$. Color identifies
    a particle, fixed across panels. DPS particles jump across the line at every step, while the
    implicit particles settle onto it.}
    \label{fig:instability_iterations}
\end{figure}

\emph{Implicit-Guidance-DPS.} Backward Euler on the same drift over the same step evaluates
$\nabla f$ at the new point and the new noise level, with the schedule held at its step-$i$
value $\eta_i=1/\Delta t_i$ (it stands in for that step's missing $\Delta t_i$):
\[
Y_{i-1}=\tilde{Y}_{i-1}-\Delta t_i\,\eta_i\,\nabla f_{i-1}(Y_{i-1})=\tilde{Y}_{i-1}-\nabla f_{i-1}(Y_{i-1})
\quad\Longleftrightarrow\quad
Y_{i-1}=\operatorname{prox}_{f_{i-1}}(\tilde{Y}_{i-1}).
\]
Since $f_{i-1}=\alpha\,|y-a^\top\hat{x}_{t_{i-1}}|$, this is the prox of the residual magnitude with
parameter $\alpha$, the same at every step: as in DPS, the step size and the schedule cancel.
The prior score stays in the explicit DDPM step; only the guidance is implicit. Because the
prior is Gaussian, $\hat{x}_{t_{i-1}}$ is affine with a constant Jacobian $J_{i-1}$, so $f_{i-1}$ is
$\alpha$ times the absolute value of an affine function, and its prox (the unique solution of
the implicit equation) is a soft-threshold of the residual at level $\alpha\|v\|^2$ along
$v=J_{i-1}^\top a$. With
$\tilde{r}=y-a^\top\hat{x}_{t_{i-1}}(\tilde{Y}_{i-1})$,
\[
Y_{i-1}=\tilde{Y}_{i-1}+\lambda\,v,\qquad
\lambda=\operatorname{sign}(\tilde{r})\,\min\!\big(\alpha,\ |\tilde{r}|/\|v\|^2\big).
\]
Far from the constraint this is the same $\alpha$-sized step as DPS; within reach it sets the
residual to exactly zero rather than overshooting. At the last step $\hat{x}=x$ and $J=I$, so the
output satisfies $x_2=y$ exactly; \cref{fig:instability_iterations}(b) shows the same three particles
under this rule.

Both rules are run with the same seed and noise, so they differ only in how the guidance is
integrated. Under DPS (\cref{fig:instability_iterations}(a)) the particles alternate sides of the
line at every step with $x_1$ fixed, and end at $x_2\in\{1.73,\,1.36,\,1.67\}$: the period-$2$
orbit predicted above. Under implicit integration (\cref{fig:instability_iterations}(b)) they
descend onto the line and stay, with residual exactly zero in float32 at the output.
\Cref{fig:instability_samples} repeats the comparison on the final states of $10^4$ particles.

\begin{figure}[!tb]
    \centering
    \includegraphics[width=\linewidth]{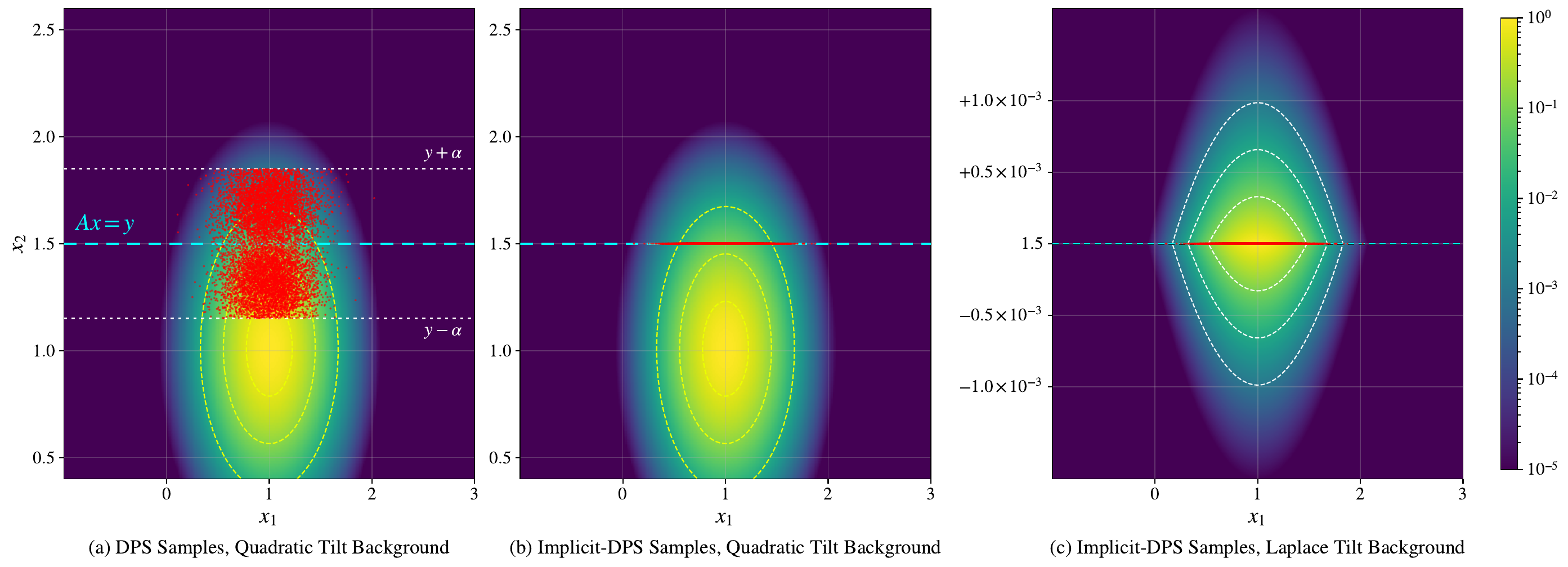}
    \caption{Final states of $10^4$ particles, same setup as \cref{fig:instability_iterations}.
    \textbf{(a)} DPS and \textbf{(b)} implicitly integrated DPS, over the density
    $\propto e^{-\alpha(x_2-y)^2/2}\rho_*$ (log scale) with its $1,2,3\sigma$ contours dashed; the
    dotted lines in (a) mark $y\pm\alpha$. \textbf{(c)} A different background: the samples of (b)
    over the Laplace tilt $\mu^{\mathrm{Target}}_y$ (shown as a fraction of its maximum, log scale),
    on the same $x_1$ range but magnified $669\times$ in $x_2$, with contours at whole
    decades. DPS spreads its samples over the band $y\pm\alpha$; the implicit rule places every
    sample on the line, where $\mu^{\mathrm{Target}}_y$ concentrates.}
    \label{fig:instability_samples}
\end{figure}

Across $10^4$ particles (\cref{fig:instability_samples}), the DPS orbit, stopped at a random
phase, fills a band of half-width $\alpha$ about the line, while the implicit rule places every
sample on it. The guidance never moves $x_1$, which keeps its prior law $\mathcal{N}(1,\sigma^2)$,
as under $\mu^{\mathrm{Target}}_y$. The entire discrepancy between the two rules is therefore the
$\pm\alpha$ oscillation in $x_2$.

\FloatBarrier

\textbf{Conditional Guidance for MNIST digits}\;
We consider the setting of posterior sampling with the MNIST prior. This dataset consists of paired images of handwritten digits and their corresponding labels, denoted $(x, y)$. We train a simple MLP classifier $\text{softmax}(f(x))\approx \mathbbm{1}_{y}$ over the MNIST dataset, and set the reward to be $R(x) = -\Vert f(x)-\mathbbm{1}_{k}\Vert$ for some fixed target $k$. We run DPS with guidance schedule constant at $0.1$. The evolution of $\Vert f(x_t)-\mathbbm{1}_k\Vert$ as well as $(f(x_t)-\mathbbm{1}_k)\cdot \mathbbm{1}$ is plotted below.
\begin{figure}[H]
    \centering
    \begin{minipage}[t]{0.48\textwidth}
        \centering
        \includegraphics[width=\textwidth]{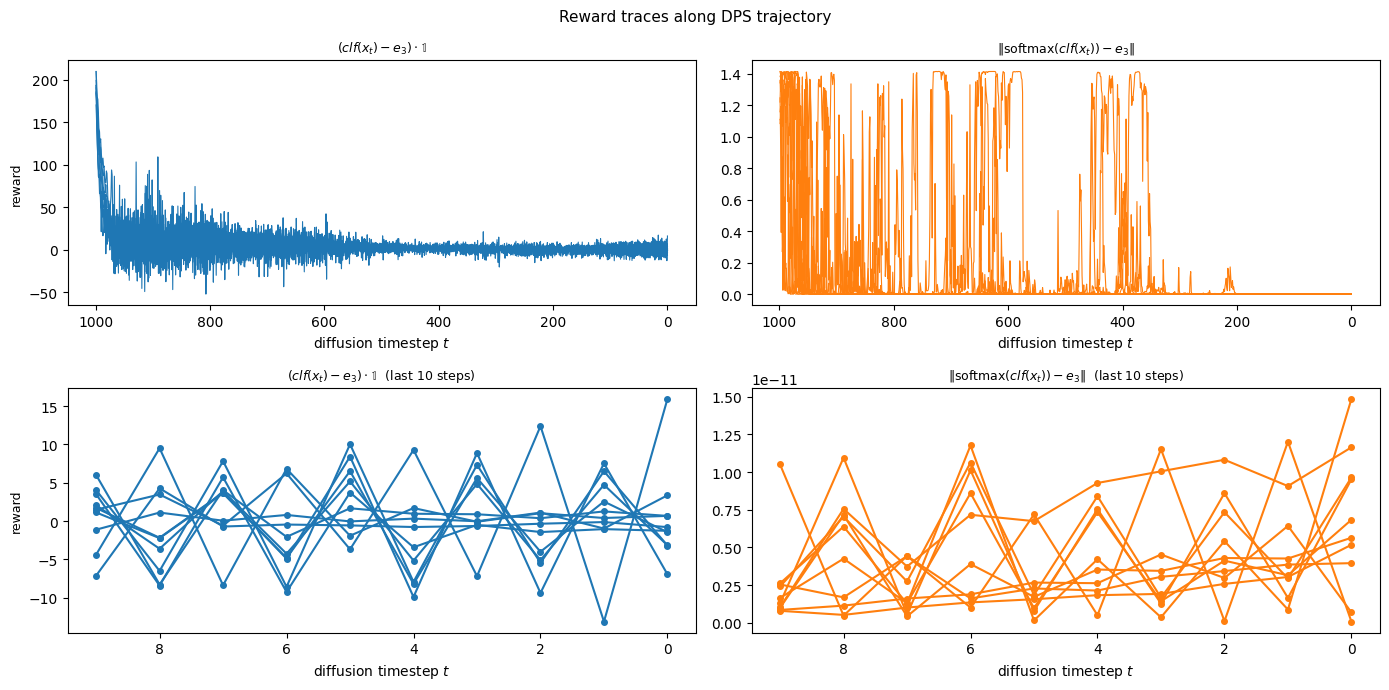}
    \end{minipage}
    \hfill
    \begin{minipage}[t]{0.48\textwidth}
        \centering
        \includegraphics[width=\textwidth]{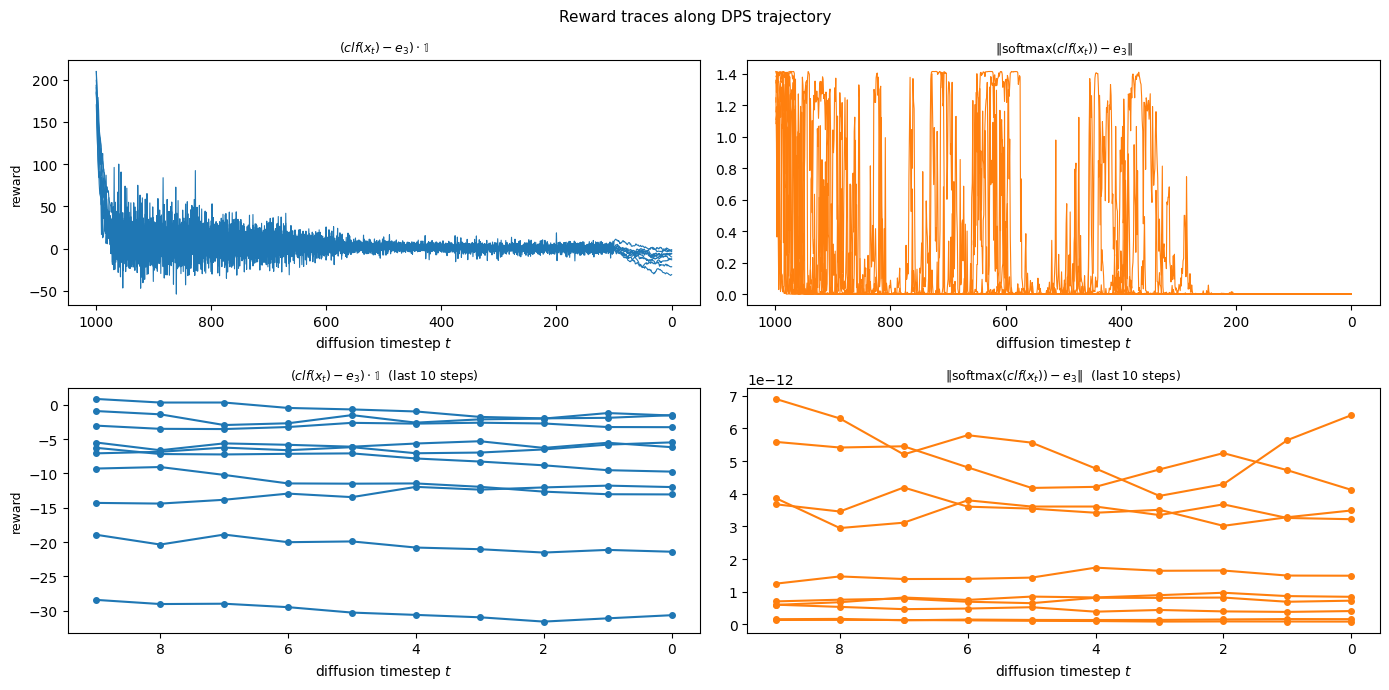}
    \end{minipage}
    \caption{We plot a projected discrepancy $(f(x_t)-\mathbbm{1}_k)\cdot \mathbbm{1}$ (\textbf{(Columns 1 and 3)}) and the distance $\Vert f(x_t)-\mathbbm{1}_k\Vert=-R(x_t)$ (\textbf{(Columns 2 and 4)}) across $t$ where denoising proceeds from left (most noise) to right (least noise). The second row depicts a close-up plot of just the last $10$ steps to highlight the oscillations. \textbf{(Columns 1 and 2)} are run with a constant guidance schedule \cref{alg:dps_gauss}, while \textbf{(Column 3 and 4)} are run with early guidance stopping \cref{alg:early_stopping} with parameter $i_{\mathrm{stop}}=100$. }
    \label{fig:instability_detailed}
\end{figure}

We see a distinct oscillatory pattern is sustained throughout the trajectory when the guidance schedule is constant \cref{alg:dps_gauss}. Turning off the guidance schedule \cref{alg:early_stopping} at time-step $i_{\mathrm{stop}}=100$ eliminates the oscillations in that period, though now the distance $\Vert f(x_t)-\mathbbm{1}_k\Vert$ is not pulled toward $0$. This indicates that the instability is associated with the reward guidance. We see in either case that a softmax applied to the logits of the classifier results in a very high confidence prediction of the correct class despite these oscillations.

We also plot the alignment between consecutive steps of the algorithm $\delta_t = x_t-x_{t-1}$. Because these vectors lie in $784$ dimensions, to emphasize the step over step alignment we maintain a subspace described by the most recent $\ell=50$ such steps. In particular, let $\mathcal{P}_t = [\delta_{t-\ell+1}, \delta_{t-\ell+2}, \cdots \delta_t]\in \mathbb{R}^{784\times \ell}$, and let $\mathcal{P}_t^{:k}$ denote just the projection onto the top $k$ principal axes. We plot $\alpha_t = (\mathcal{P}_t^{:k}\delta_{t-1})\cdot(\mathcal{P}_t^{:k}\delta_t)$ over the trajectory in \cref{fig:instability_step_to_step}. For a purely oscillating trajectory, we expect $\delta_t \approx -\delta_{t-1}$, resulting in $\alpha_t \approx -1$. When $\delta_t$ is ``unrelated'' to $\delta_{t-1}$, we expect $\alpha_t \approx 0$.

\begin{figure}[H]
    \centering
    \includegraphics[width=0.8\textwidth]{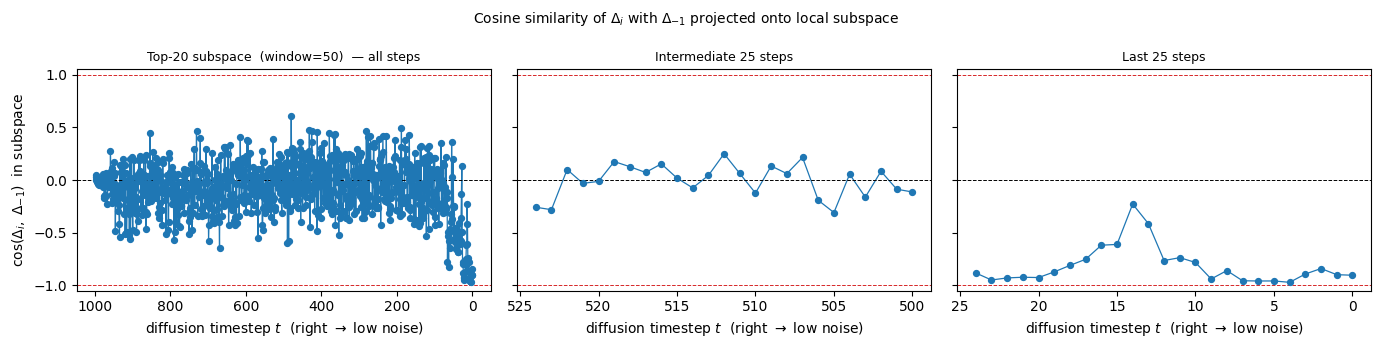}\\[1em]
    \includegraphics[width=0.8\textwidth]{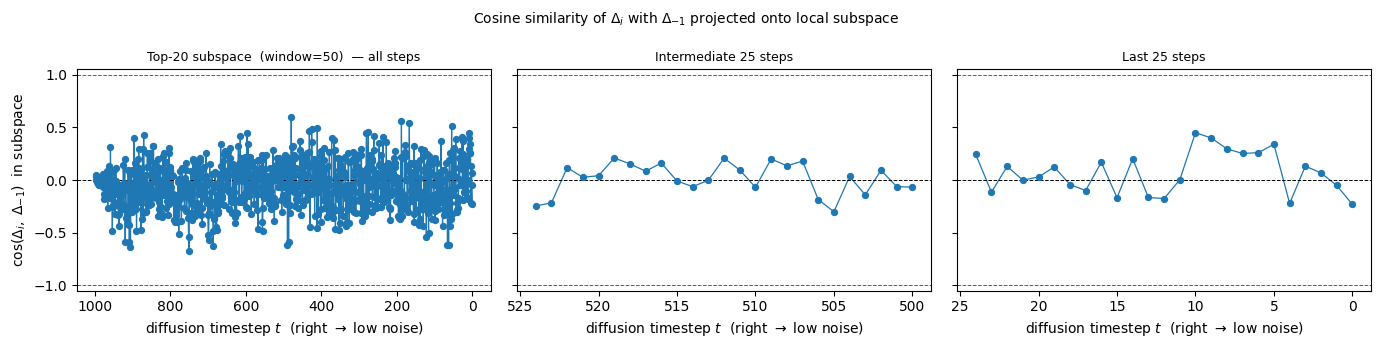}
    \caption{\textbf{Top} Are plots associated to the standard DPS algorithm \cref{alg:dps_gauss},
    \textbf{Top Left} We plot $\alpha_t = (\mathcal{P}_t^{:k}\delta_t)\cdot(\mathcal{P}_t^{:k}\delta_{t-1})$ along the DPS trajectories for a constant guidance schedule $\zeta=0.1$. \textbf{Top Middle} A close-up of steps $525\to 500$. \textbf{Top Right} A close-up of steps $25\to 0$. Note that $\alpha_t$ is close to $0$ at the intermediate noise levels, but drops to $\approx -1$ towards the low noise levels. \textbf{Bottom} Are the same plots associated to \cref{alg:early_stopping} with $i_{\mathrm{stop}}=100$.
    We observe that $\alpha_t$ remains close to $0$ both at intermediate noise levels and low noise levels.}
    \label{fig:instability_step_to_step}
\end{figure}

All experiments were run in a few minutes on a single NVIDIA H100 GPU.

\end{document}